%% file: AIJ_HIW_BPARHMM.tex
\documentclass[preprint,12pt]{elsarticle}

\usepackage{amsmath}
\usepackage{amsfonts}
\usepackage{amsthm}
\usepackage{amssymb}

\usepackage{lineno}

\usepackage{booktabs}
\usepackage{algorithm}
\usepackage{algorithmicx}
\usepackage{algpseudocode}
\usepackage{fancyhdr}

\input{etc/mathShortcuts}

\journal{Artificial Intelligence Journal}

\begin{document}

\begin{frontmatter}



\title{Modeling the Complex Dynamics and Changing Correlations of Epileptic Events}


\author[PennBE]{Drausin F. Wulsin}
\author[UWStat]{Emily B. Fox}
\author[PennBE,PennNeur]{Brian Litt}

\address[PennBE]{Department of Bioengineering, University of Pennsylvania, Philadelphia, PA}
\address[PennNeur]{Department of Neurology, University of Pennsylvania, Philadelphia, PA}
\address[UWStat]{Department of Statistics, University of Washington, Seattle, WA}

\begin{abstract}
Patients with epilepsy can manifest short, sub-clinical epileptic ``bursts'' in addition to full-blown clinical seizures. We believe the relationship between these two classes of events---something not previously studied quantitatively---could yield important insights into the nature and intrinsic dynamics of seizures. A goal of our work is to parse these complex epileptic events into distinct dynamic regimes.  A challenge posed by the intracranial EEG (iEEG) data we study is the fact that the number and placement of electrodes can vary between patients.  We develop a Bayesian nonparametric Markov switching process that allows for (i) shared dynamic regimes between a variable number of channels, (ii) asynchronous regime-switching, and (iii) an unknown dictionary of dynamic regimes.  We encode a sparse and changing set of dependencies between the channels using a Markov-switching Gaussian graphical model for the innovations process driving the channel dynamics and demonstrate the importance of this model in parsing and out-of-sample predictions of iEEG data.  We show that our model produces intuitive state assignments that can help automate clinical analysis of seizures and enable the comparison of sub-clinical bursts and full clinical seizures.
\end{abstract}


\end{frontmatter}


\section{Introduction}

Despite over three decades of research, we still have very little idea of what defines a seizure. This ignorance stems both from the complexity of epilepsy as a disease and a paucity of quantitative tools that are flexible enough to describe epileptic events but restrictive enough to distill intelligible information from them. Much of the recent machine learning work in electroencephalogram (EEG) analysis has focused on seizure prediction, \citep[cf.,][]{D'Alessandro2005,Mirowski2009}, an important area of study but one that generally has not focused on parsing the EEG directly, as a human EEG reader would. Such parsings are central for diagnosis and relating various types of abnormal activity. Recent evidence shows that the range of epileptic events extends beyond clinical seizures to include shorter, sub-clinical ``bursts'' lasting fewer than 10 seconds \cite{Litt2001}. What is the relationship between these shorter bursts and the longer seizures?  In this work, we demonstrate that machine learning techniques can have substantial impact in this domain by unpacking how seizures begin, progress, and end.  

In particular, we build a Bayesian nonparametric time series model to analyze intracranial EEG (iEEG) data. 
We take a modeling approach similar to a physician's in analyzing EEG events: look directly at the evolution of the raw EEG voltage traces. EEG signals exhibit nonstationary behavior during a variety of neurological events, and time-varying autoregressive (AR) processes have been proposed to model single channel data \cite{Krystal1999}.  Here we aim to parse the recordings into interpretable regions of activity and thus propose to use autoregressive hidden Markov models (AR-HMMs) to define \emph{locally} stationary processes.  In the presence of multiple channels of simultaneous recordings, as is almost always the case in EEG, we wish to share AR states between the channels while allowing for asynchronous switches. The recent beta process (BP) AR-HMM of~\citet{Fox2009} provides a flexible model of such dynamics: a shared library of infinitely many possible AR states is defined and each time series uses a finite subset of the states.  The process encourages sharing of AR states, while allowing for time-series-specific variability.  

Conditioned on the selected AR dynamics, the BP-AR-HMM assumes independence between time series. In the case of iEEG, this assumption is almost assuredly false. Figure~\ref{fig:electSpatSetup} shows an example of a 4x8 intracranial electrode grid and the residual EEG traces of 16 channels \emph{after} subtracting the predicted value in each channel using a conventional BP-AR-HMM. While the error term in some channels remains low throughout the recording, other channels---especially those spatially adjacent in the electrode grid---have very correlated error traces. We propose to capture correlations between channels by modeling a multivariate innovations process that drives independently evolving channel dynamics.  We demonstrate the importance of accounting for this error trace in predicting heldout seizure recordings, making this a crucial modeling step before undertaking large-scale EEG analysis.  
\begin{figure*}[t]
	\begin{center}
		\includegraphics[width=\textwidth]{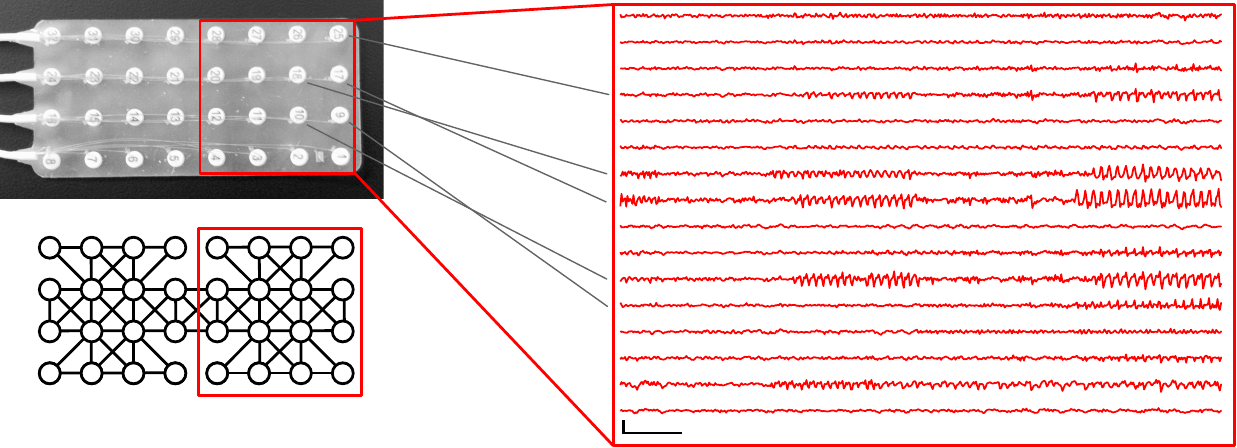} 
	  \end{center}
	\vspace{-0.1in}\caption{An iEEG grid electrode and (\textbf{bottom left}) corresponding graphical model. (\textbf{middle}) Residual EEG values \emph{after} subtracting predictions from a BP-AR-HMM assuming independent channels. All EEG scale bars indicate 1 mV vertically and 1 second horizontally.
	}
	\label{fig:electSpatSetup}
\end{figure*} 

To aid in scaling to large electrode grids, we exploit a sparse dependency structure for the multivariate innovations process. In particular, we assume a graph with known vertex structure that encodes conditional independencies in the multivariate innovations process.  The graph structure is based on the spatial adjacencies of the iEEG channels, with a few exceptions to make the graphical model fully decomposable.  Figure~\ref{fig:electSpatSetup} (bottom left) shows an example of such a graphical model over the channels. Although the relative position of channels in the electrode grid is clear, determining the precise 3D location of each channel is extremely difficult. Furthermore, unlike in scalp EEG or magentoencephalogram (MEG), which have generally consistent channel positions from patient to patient, iEEG channels vary in number and position for each patient.  These issues impede the use of alternative spatial and multivariate time series modeling techniques.   

It is well-known that the correlations between EEG channels usually vary during the beginning, middle, and end of a seizure \cite{Schindler2007,Schiff2005}.  \citet{Prado2006} employ a mixture-of-expert vector autoregressive (VAR) model to describe the different dynamics present in seven channels of scalp EEG.  We take a similar approach by allowing for a Markov evolution to an underlying innovations covariance state.

An alternative modeling approach is to treat the channel recordings as a single multivariate time series, perhaps using a switching VAR process as in~\citet{Prado2006}.  However, such an approach (i) assumes synchronous switches in dynamics between channels, (ii) scales poorly with the number of channels, and (iii) requires an identical number of channels between patients to share dynamics between event recordings. 

Other work has explored nonparametric modeling of multiple time series. The infinite factorial HMM of \citet{VanGael2008} considers an infinite collection of chains each with a binary state space.  The infinite hierarchical HMM \cite{Heller2009} also involves infinitely many chains with finite state spaces, but with constrained transitions between the chains in a top down fashion.  The infinite DBN of \citet{Doshi-Velez2011} considers more general connection structures and arbitrary state spaces.  Alternatively, the graph-coupled HMM of \citet{Dong2012} allows graph-structured dependencies in the underlying states of some $N$ Markov chains. Here, we consider a finite set of chains with infinite state spaces that evolve independently. The factorial structure combines the chain-specific AR dynamic states and the graph-structured innovations to generate the multivariate observations with sparse dependencies.

Expanding upon previous work \cite{Wulsin2013}, we show that our model for correlated time series has better out-of-sample predictions of iEEG data than standard AR- and BP-AR-HMMs and demonstrate the utility of our model in comparing short, sub-clinical epileptic bursts with longer, clinical seizures.  Our inferred parsings of iEEG data concur with key features hand-annotated by clinicians but provide additional insight beyond what can be extracted from a visual read of the data.  The importance of our methodology is multifold: (i) the output is interpretable to a practitioner and (ii) the parsings can be used to relate seizure types both within and between patients even with different electrode setups. Enabling such broad-scale automatic analysis, and identifying dynamics unique to sub-clinical seizures, can lead to new insights in epilepsy treatments.

Although we are motivated by the study of seizures from iEEG data, our work is much more broadly applicable in time series analysis.  For example, perhaps one has a collection of stocks and wants to model shared dynamics between them while capturing changing correlations.  The BP-AR-HMM was applied to the analysis of a collection of motion capture data assuming independence between individuals; our modeling extension could account for coordinated motion with a sparse dependency structure between individuals.  Regardless, we find the impact in the neuroscience domain to be quite significant.

\section{A Structured Bayesian Nonparametric Factorial AR-HMM}

\subsection{Dynamic Model}
Consider an event with $N$ univariate time series of length $T$. This event could be a seizure, where each time series is one of the iEEG voltage-recording channels. For clarity of exposition, we refer to the individual univariate time series as \emph{channels} and the resulting $N$-dimensional multivariate time series (stacking up the channel series) as the \emph{event}. We denote the scalar value for each channel $i$ at each (discrete) time point $t$ as $y^{(i)}_t$ and model it using an $r$-order AR-HMM~\cite{Fox2009}.  That is, each channel is modeled via Markov switches between a set of AR dynamics.  Denoting the latent state at time $t$ by $z_t^{(i)}$, we have:
\begin{align}
\begin{aligned}
	z_t^{(i)} & \sim \mb{\pi}^{(i)}_{z_{t-1}^{(i)}}, \\
	y_t^{(i)} & = \sum_{j=1}^r a_{z_t^{(i)},j} y_{t-j}^{(i)} + \epsilon_t^{(i)} = \mb{a}^\mathrm{T}_{z_t^{(i)}}\mb{\widetilde{y}}_{t}^{(i)} + \epsilon_t^{(i)}.	
\end{aligned}
\label{eq:channelARHMM}
\end{align}
Here, $\mb{a}_{k} = \left(a_{k,1},\dots,a_{k,r}\right)^\T$ are the AR parameters for state $k$ and $\pi_k$ is the transition distribution from state $k$ to any other state.  We also introduce the notation $\widetilde{\mb{y}}_t^{(i)}$ as the vector of $r$ previous observations  $\left(y_{t-1}^{(i)},\dots,y_{t-r}^{(i)}\right)^\T$. 

In contrast to a vector AR (VAR) HMM specification of the event, our modeling of channel dynamics separately as in Eq.~\eqref{eq:channelARHMM} allows for (i) asynchronous switches and (ii) sharing of dynamic parameters between recordings with a potentially different number of channels.  However, a key aspect of our data is the fact that the channels are correlated.  Likewise, these correlations change as the patient progresses through various seizure \emph{event states} (e.g., ``resting'', ``onset'', ``offset'', \dots).  That is, the channels may display one innovation covariance before a seizure (e.g., relatively independent and low-magnitude) but quite a different covariance during a seizure (e.g., correlated, higher magnitude).  To capture this, we jointly model the innovations $\mb{\epsilon}_t = \left(\epsilon_t^{(1)}, \dots, \epsilon_t^{(N)}\right)^{\rm T}$ driving the AR-HMMs of Eq.~\eqref{eq:channelARHMM} as
\begin{align}
\begin{aligned}
	Z_t & \sim \mb{\phi}_{Z_{t-1}}, \\
	\mb{\epsilon}_t & \sim \Ncal(\mb{0}, \Delta_{Z_t}),
	\label{eq:eventHMM}
\end{aligned}
\end{align}
where $Z_t$ denotes a Markov-evolving event state distinct from the individual channel states $\{z_t^{(i)}\}$, $\mb{\phi}_{l}$ the transition distributions, and $\Delta_k$ the event-state-specific channel covariance.  That is each $\Delta_l$ describes a different set of channel relationships. 

For compactness, we sometimes alternately write
\begin{align}
	\mb{y}_t = \mb{A}_{\mb{z}_t}\widetilde{\mb{Y}}_t + \mb{\epsilon}_t(Z_t),
	\label{eq:vec_yt_def}
\end{align}
where $\mb{y}_t$ is the concatenation of $N$ channel observations at time $t$ and $\mb{z}_t$ is the vector of concatenated channel states.  The overall dynamic model is represented graphically in Figure~\ref{fig:ARHMMgraph}.

\begin{figure}[tp]
	\begin{center}
		\includegraphics[width=0.75\textwidth]{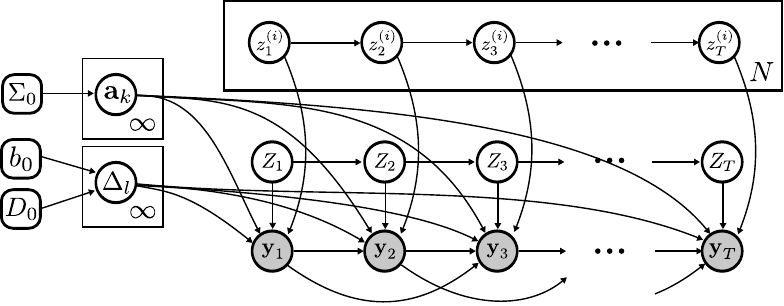}
	\end{center}
	\vspace{-0.1in}\caption{Graphical model of our factorial AR-HMM.  The $N$ channel states $z_t^{(i)}$ evolve according to independent Markov processes (transition distributions omitted for simplicity) and index the AR dynamic parameters $\mb{a}_k$ used in generating observation $y_t^{(i)}$.  The Markov-evolving event state $Z_t$ indexes the graph-structured covariance $\Delta_l$ of the correlated AR innovations resulting in multivariate observations $\mb{y}_t = [y_t^{(1)},\ldots,y_t^{(N)}]^{\rm T}$ sharing the same conditional independencies. }
	\label{fig:ARHMMgraph}
\end{figure}

\paragraph{Scaling to large electrode grids}
To scale our model to a large number of channels, we consider a Gaussian graphical model (GGM) for $\mb{\epsilon}_t$ capturing a sparse dependency structure amongst the channels.  Let $G=(V,E)$ be an undirected graph with $V$ the set of channel \emph{nodes} $i$ and $E$ the set of \emph{edges} with $(i,j) \in E$ if $i$ and $j$ are connected by an edge in the graph.  Then, $\left[\Delta_l^{-1}\right]_{ij} = 0$ for all $(i,j) \not \in E$, implying $\epsilon_t^{(i)}$ is conditionally independent of $\epsilon_t^{(j)}$ given $\epsilon_t^{(k)}$ for all channels $k \neq i,j$.  In our dynamic model of Eq.~\eqref{eq:channelARHMM}, statements of conditional independence of $\mb{\epsilon}_t$ translate directly to statements of the observations $\mb{y}_t$.

In our application, we choose $G$ based on the spatial adjacencies of channels in the electrode grid, as depicted in Figure~\ref{fig:electSpatSetup} (bottom left).  In addition to encoding the spatial proximities of iEEG electrodes, the graphical model also yields a sparse precision matrix $\Delta_l^{-1}$, allowing for more efficient scaling to the large number of channels commonly present in iEEG. These computational efficiencies are made clear in Section~\ref{sec:MCMC}.

\paragraph{Interpretation as a sparse factorial HMM}
Recall that our formulation involves $N+1$ independently evolving Markov chains: $N$ chains for the channel states $z_t^{(i)}$ plus one for the event state sequence $Z_t$.  As indicated by the observation model of Eq.~\eqref{eq:vec_yt_def}, the $N+1$ Markov chains jointly generate our observation vector $\mb{y}_t$ leading to an interpretation of our formulation as a factorial HMM \citep{Ghahramani1997}.  However, here we have a \emph{sparse} dependency structure in how the Markov chains influence a given observation $\mb{y}_t$, as induced by the conditional independencies in $\mb{\epsilon}_t$ encoded in the graph $G$.  That is, $y_t^{(i)}$ only depends on $Z_t$ the set of $z_t^{(j)}$ for which $j$ is a neighbor of $i$ in $G$.

\subsection{Prior Specification}
\label{sec:prior}

\paragraph{Emission parameters}
As in the AR-HMM, we place a multivariate normal prior on the AR coefficients,
\begin{equation}
	\mb{a}_k \sim \Ncal(\mb{m}_0, \Sigma_0),
	\label{eq:a_prior}
\end{equation}
with mean $\mb{m}_0$ and covariance $\Sigma_0$. Throughout this work, we let $\mb{m}_0 = \mb{0}$. 

For the channel covariances $\Delta_l$ with sparse precisions $\Delta_l^{-1}$ determined by the graph $G$, we specify a hyper-inverse Wishart (HIW) prior,
\begin{equation}
	\Delta_l \sim \HIW_G(b_0, D_0),
\end{equation}
where $b_0$ denotes the degrees of freedom and $D_0$ the scale. The HIW prior \cite{Dawid1993} enforces hyper-Markov conditions specified by $G$. 

\paragraph{Feature constrained channel transition distributions}
A natural question is how many AR states are the channels switching between?  Likewise, which are shared between the channels and which are unique?  We expect to see similar dynamics present in the channels (sharing of AR processes), but also some differences.  For example, maybe only some of the channels ever get excited into a certain state.  To capture this structure, we take a Bayesian nonparametric approach building on the beta process (BP) AR-HMM of~\citet{Fox2011c}.  Through the beta process prior \cite{Thibaux2007}, the BP-AR-HMM defines a shared library of infinitely many AR coefficients $\{\mb{a}_k\}$, but encourages each channel to only use a sparse subset of them.  

The BP-AR-HMM specifically defines a featural model.  Let $\mb{f}^{(i)}$ be a binary feature vector associated with channel $i$ with $f_k^{(i)}=1$ indicating that channel $i$ uses the dynamic described by $\mb{a}_k$. Formally, the feature assignments $f_k^{(i)}$ and their corresponding parameters $\mb{a}_k$ are generated by a beta process random measure and the conjugate Bernoulli process (BeP),
\begin{align}
\begin{aligned}
	B & \sim \BP(1,B_0), \\
	X^{(i)} & \sim \BeP(B),
\end{aligned}
\end{align}
with base measure $B_0$ over the parameter space $\Theta = \mathbb{R}^r$ for our $r$-order autoregressive parameters $\mb{a}_k$. As specified in Eq.~\eqref{eq:a_prior}, we take the normalized measure $B_0 / B_0(\Theta)$ to be $\Ncal(\mb{m}_0, \Sigma_0)$. The discrete measures $B$ and $X^{(i)}$ can be represented as
\begin{align}
	B = \sum_{k=1}^{\infty} \omega_k \delta_{\mb{a}_k}, \quad
	X^{(i)}  = \sum_{k=1}^\infty f^{(i)}_k \delta_{\mb{a}_k},
\end{align}
with $f^{(i)}_k \sim \Ber(\omega_k)$. 

The resulting feature vectors $\mb{f}^{(i)}$ constrain the set of available states $z_t^{(i)}$ can take by constraining each transition distributions, $\mb{\pi}^{(i)}_j$, to be 0 when $f_k^{(i)}=0$. Specifically, the BP-AR-HMM defines $\mb{\pi}^{(i)}_j$ by introducing a set of gamma random variables, $\eta_{jk}^{(i)}$, and setting
\begin{align}
	\eta^{(i)}_{jk} & \sim \GammaD(\gamma_{\rm c} + \kappa_{\rm c}\delta(j,k)) \\
	\mb{\pi}^{(i)}_j & = \frac{\mb{\eta}^{(i)}_j \circ \mb{f}^{(i)}}{\sum_{k\mid f_k^{(i)} = 1}\eta^{(i)}_{jk}} \label{eq:pi_eta_f}.
\end{align}
The positive elements of $\mb{\pi}^{(i)}_j$ can also be thought of as a sample from a finite Dirichlet distribution with only $K^{(i)}$ dimensions, where $K^{(i)} = \sum_k f^{(i)}_k$ represents the number of states channel $i$ uses. For convenience, we sometimes denote the set of transition variables $\{\eta_{jk}^{(i)}\}_j$ as $\mb{\eta}^{(i)}$. As in the sticky HDP-HMM of \citet{Fox2011}, the parameter $\kappa_c$ encourages self-transitions (i.e., state $j$ at time $t-1$ to state $j$ at time $t$). 

\paragraph{Unconstrained event transition distributions}
We again take a Bayesian nonparametric approach to define the event state HMM, building on the sticky HDP-HMM \cite{Fox2011}. In particular, the transition distributions $\mb{\phi}_l$ are hierarchically defined as
\begin{align}
\begin{aligned} \label{eq:stickyHDP}
	\mb{\beta} &\sim \mbox{stick}(\alpha), \\
	\mb{\phi}_l &\sim \mbox{DP}(\alpha_{\rm e}\mb{\beta} + \kappa_{\rm e}\mb{e}_l),
\end{aligned}
\end{align}
where $\mbox{stick}(\alpha)$ refers to a stick-breaking measure, also known as $\mbox{GEM}(\alpha)$, with $\mb{\beta}$ generated by
\begin{align}
\begin{aligned}
    \beta_k' & \sim \BetaD(1, \alpha), & \qquad k=1,2,\ldots, \\
    \beta_k & = \beta_k' \prod_{\ell=1}^{k-1} (1 - \beta_\ell'), & \qquad k=1,2,\ldots, \\
    & = \beta_k' \left(1 - \sum_{\ell=1}^{k-1}\beta_\ell \right) & \qquad k=1,2,\ldots.
\end{aligned}
\end{align}
Again, the sticky parameter $\kappa_{\rm e}$ promotes self-transitions, reducing state redundancy.

We term this model the \emph{sparse factorial BP-AR-HMM}. Although the graph $G$ can be arbitrarily structured, because of our motivating seizure modeling application with a focus on a spatial-based graph structure, we often describe the sparse factorial BP-AR-HMM as capturing \emph{spatial} correlations. We depict this model in the directed acyclic graphs shown in Figure~\ref{fig:DAGs_HIW_BPARHMM}. Note that while we formally consider a model of only a single event for notational simplicity, our formulation scales straightforwardly to multiple independent events. In this case, everything except the library of AR states $\{\mb{a}_k\}$ becomes event-specific.  If all events share the same channel setup, we can assume the channel covariances $\{\Delta_l\}$ are shared as well.

\begin{figure}[tp]
	\begin{center}
		\includegraphics[width=0.75\textwidth]{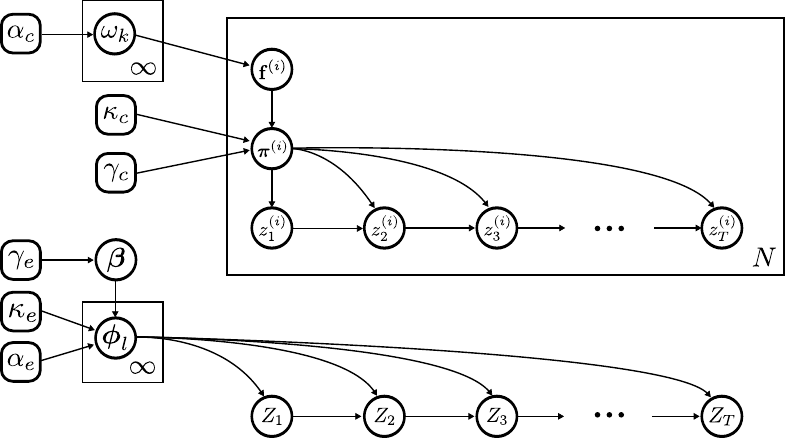}
	\end{center}
	\vspace{-0.1in}\caption[Sparse factorial BP-AR-HMM graphical models]{Referencing the channel state and event state sequences of Figure~\ref{fig:ARHMMgraph}, here we depict the graphical model associated with our Bayesian nonparametric prior specification of Section~\ref{sec:prior}. (\textbf{top}) The channel $i$ feature indicators $\mb{f}^{(i)}$ are samples from a Bernoulli process with weights $\{\omega_k\}$ and constrain the channel transition distributions $\mb{\pi}^{(i)}$. Channel states $z_t^{(i)}$ evolve independently for each channel according to these feature-constrained transition distributions $\mb{\pi}^{(i)}$.  (\textbf{bottom}) The event state $Z_t$ evolves independently of each channel $i$'s state $z_t^{(i)}$ according to transition distributions $\{\mb{\phi}_l\}$, which are coupled by global transition distribution $\mb{\beta}$.}
	\label{fig:DAGs_HIW_BPARHMM}
\end{figure}

\section{Posterior Computations}
\label{sec:MCMC}

Although the components of our model related to the individual channel dynamics are similar to those in the BP-AR-HMM, our posterior computations are significantly different due to the coupling of the Markov chains via the correlated innovations $\mb{\epsilon}_t$. In the BP-AR-HMM, conditioned on the feature assignments, each time series is independent.  Here, however, we are faced with a factorial HMM structure and the associated challenges. Yet the underlying graph structure of the channel dependencies mitigates the scale of these challenges. 

Conditioned on channel sequences $\{\mb{z}_{1:T}^{(\mb{i}')}\}$, we can marginalize $z_{1:T}^{(i)}$; because of the graph structure, we need only condition on a \emph{sparse} set of other channels $\mb{i}'$ (i.e., neighbors of channel $i$ in the graph). This step is important for efficiently sampling the feature assignments $\mb{f}^{(i)}$. 

\begin{algorithm}[tp]
	\caption{Sparse factorial BP-AR-HMM master MCMC sampler}
 	\label{alg:HIW_BPARHMM_main}
 	\begin{small}
 	\begin{algorithmic}[1]
		\For{ each MCMC iteration }
		\State get a random permutation $\mb{h}$ of the channel indices, 
		\For {each channel $i \in \mb{h} $}
		\State sample feature indicators $\mb{f}^{(i)}$ as in Eq.~\eqref{eq:f_post}
		\State sample state sequence $z_{1:T}^{(i)}$ as in Eq.~\eqref{eq:z_post}
		\State sample state transition parameters $\mb{\eta}^{(i)}$ as in Eq.~\eqref{eq:eta_post}
		\EndFor
		\State sample event states sequence $Z_{1:T}$
		\State sample event state transition parameters $\mb{\phi}$ as in Eq.~\eqref{eq:phi_post} 
		\State sample channel AR parameters $\{\mb{a}_k\}$ as in Eq.~\eqref{eq:chStateARCoefsPost}
		\State sample channel $\{\Delta_l\}$ as in Eq.~\eqref{eq:evStateCovPost}
		\State (sample hyperparameters $\gamma_c$, $\kappa_c$, $\alpha_e$, $\kappa_e$, $\gamma_e$, and $\alpha_c = B_0(\Theta)$)
		\EndFor
	\end{algorithmic}
	\end{small}
\end{algorithm}

At a high level, each MCMC iteration proceeds through sampling channel states, events states, dynamic model parameters, and hyperparameters. Algorithm~\ref{alg:HIW_BPARHMM_main} summarizes these steps, which we briefly describe below and more fully in Appendices B-D.

\paragraph{Individual channel variables} 
We harness the fact that we can compute the marginal likelihood of $\mb{y}_{1:T}$ given $\mb{f}^{(i)}$ and the neighborhood set of other channels $\mb{z}_{1:T}^{(\mb{i}')}$ in order to block sample $\{\mb{f}^{(i)}, z_{1:T}^{(i)}\}$.  That is, we first sample $\mb{f}^{(i)}$ marginalizing $z_{1:T}^{(i)}$ and then sample $z_{1:T}^{(i)}$ given the sampled $\mb{f}^{(i)}$.  
Sampling the active features $\mb{f}^{(i)}$ for channel $i$ follows as in~\citet{Fox2009}, using the Indian buffet process (IBP) \cite{Griffiths2005} predictive representation associated with the beta process, but using a likelihood term that conditions on neighboring channel state sequences $\mb{z}_{1:T}^{(\mb{i}')}$ and observations $\mb{y}_{1:T}^{(\mb{i}')}$.  We additionally condition on the event state sequence $Z_{1:T}$ to define the sequence of distributions on the innovations.  Generically, this yields 
\begin{multline}  \label{eq:f_post}
	p\left(f^{(i)}_k \mid y^{(i)}_{1:T}, \mb{y}^{(\mb{i}')}_{1:T}, \mb{z}^{(\mb{i}')}_{1:T}, Z_{1:T}, \mb{F}^{-ik}, \mb{\eta}^{(i)}, \{\mb{a}_k\}, \{\Delta_l\}\right) \propto \\
	p\left(f^{(i)}_k \mid \mb{F}^{-ik} \right)p\left(y^{(i)}_{1:T} \mid \mb{y}^{(\mb{i}')}_{1:T}, \mb{z}^{(\mb{i}')}_{1:T}, Z_{1:T}, \mb{F}^{-ik}, f^{(i)}_k, \mb{\eta}^{(i)}, \{\mb{a}_k\}, \{\Delta_l\} \right).
\end{multline}
Here, $\mb{F}^{-ik}$ denotes the set of feature assignments not including $f^{(i)}_k$.  The first term is given by the IBP prior and the second term is the marginal conditional likelihood (marginalizing $z_{1:T}^{(i)}$). Based on the derived marginal conditional likelihood, feature sampling follows similarly to that of~\citet{Fox2009}.

Conditioned on $\mb{f}^{(i)}$, we block sample the state sequence $z_{1:T}^{(i)}$ using a backward filtering forward sampling algorithm (see \ref{apndx:hmmSumProd})based on a decomposition of the full conditional as
\begin{multline} \label{eq:z_post}
  p\left(z^{(i)}_{1:T} \mid y^{(i)}_{1:T}, \mb{y}^{(\mb{i}')}_{1:T}, \mb{z}^{(\mb{i}')}_{1:T}, \mb{f}^{(i)}, \mb{\eta}^{(i)},\{\mb{a}_k\},\{\Delta_l\}\right) = \\
	p\left(z^{(i)}_{1} \mid y^{(i)}_{1}, \mb{y}^{(\mb{i}')}_{1}, \mb{z}^{(\mb{i}')}_{1}, \mb{f}^{(i)}, \mb{\eta}^{(i)}, \{\mb{a}_k\},\{\Delta_l\}\right)\cdot\\
	\prod_{t=2}^T p\left(z^{(i)}_{t} \mid y^{(i)}_{t:T}, \mb{y}^{(\mb{i}')}_{t:T}, z^{(i)}_{t-1}, \mb{z}^{(\mb{i}')}_{t:T},\mb{f}^{(i)}, \mb{\eta}^{(i)}. \{\mb{a}_k\},\{\Delta_l\}\right).
\end{multline}

For sampling the transition parameters $\mb{\eta}^{(i)}$, we follow \citet[Supplement]{Hughes2012} and sample from the full conditional
\begin{align} \label{eq:eta_post}
	p(\eta^{(i)}_{jk} \mid z^{(i)}_{1:T}, f^{(i)}_k) \propto \frac{(\eta_{jk}^{(i)})^{n^{(i)}_{jk} + \gamma_{\rm c} + \kappa_{\rm c}\delta(j,k) - 1}e^{\eta^{(i)}_{jk}}}{\sum_{k'\mid f^{(i)}_k = 1} \eta^{(i)}_{jk'}},
\end{align}
where $n^{(i)}_{jk}$ denotes the number of times channel $i$ transitions from state $j$ to state $k$. We sample $\mb{\eta}^{(i)}_{j}  = C^{(i)}_j \bar{\mb{\eta}}^{(i)}_j$ from its posterior via two auxiliary variables,
\begin{align}
\begin{aligned}
	\bar{\mb{\eta}}^{(i)}_j & \sim \Dir(\gamma_{\rm c} + \mb{e}_j\kappa_{\rm c} + \mb{n}_j^{(i)}) \\
	C_{j}^{(i)} & \sim \mbox{Gamma}(K\gamma_{\rm c} + \kappa_{\rm c}, 1),
\end{aligned}
\end{align}
where $\mb{n}_j^{(i)}$ gives the transition counts from state $j$ in channel $i$.  

\paragraph{Event variables $\{\mb{\phi}_l,\Delta_l,Z_{1:T}\}$} 
Conditioned on the channel state sequences $\mb{z}_{1:T}$ and AR coefficients $\{\mb{a}_k\}$, we can compute an innovations sequence as $\mb{\epsilon}_t = \mb{y}_t - \mb{A}_{\mb{z}_t}\mb{\widetilde{Y}}_t$, where we recall the definition of $\mb{A}_{k}$ and $\mb{\widetilde{Y}}_t$ from Eq.~\eqref{eq:vec_yt_def}.  These innovations are the observations of the sticky HDP-HMM of Eq.~\eqref{eq:eventHMM}.  For simplicity and to allow block-sampling of $\mb{z}_{1:T}$, we consider a weak limit approximation of the sticky HDP-HMM as in~\cite{Fox2011}. The top-level Dirichlet process is approximated by an $L$-dimensional Dirichlet distribution~\cite{Ishwaran2002}, inducing a finite Dirichlet for $\mb{\phi}_l$:
\begin{align}
\begin{aligned} \label{eq:truncatedHDP}
	\mb{\beta} &\sim \Dir(\gamma_{\rm e}/L, \ldots, \gamma_{\rm e}/L), \\
	\mb{\phi}_l &\sim \Dir(\alpha_{\rm e}\mb{\beta} + \kappa_{\rm e}\mb{e}_l).
\end{aligned}
\end{align}
Here, $L$ provides an upper bound on the number of states in the HDP-HMM.  The weak limit approximation still encourages using a subset of these $L$ states.

Based on the weak limit approximation, we first sample the parent transition distribution $\mb{\beta}$ as in~\cite{Teh2006,Fox2011}, followed by sampling each $\mb{\phi}_l$ from its Dirichlet posterior,
\begin{align} \label{eq:phi_post}
	p\left(\mb{\phi}_l \mid Z_{1:T}, \mb{\beta}\right) \propto \Dir(\alpha_{\rm e}\mb{\beta} + \mb{e}_l\kappa_{\rm e} + \mb{n}_l),
\end{align}
where $\mb{n}_l$ is a vector of transition counts of $Z_{1:T}$ from state $l$ to the $L$ different states. 

Using standard conjugacy results, based on ``observations'' $\mb{\epsilon}_t = \mb{y}_t - \mb{A}_{\mb{z}_t}\mb{\widetilde{Y}}_t$ for $t$ such that $Z_t = l$, the full conditional for $\Delta_l$ is given by
\begin{equation}
	p(\Delta_l \mid \mb{y}_{1:T}, \mb{z}_{1:T}, Z_{1:T}, \{\mb{a}_k\}) \propto \HIW_G(b_l, D_l),
	\label{eq:evStateCovPost}
\end{equation}
where
\begin{align*}
	b_l & = b_0 + |\{t \mid Z_t = l,\; t=1,\ldots,T \}|, \\
	D_l & = D_0 + \sum_{t \mid Z_t = l} \mb{\epsilon}_t\mb{\epsilon}_t^\T.
\end{align*}
Details on how to efficiently sample from a HIW distribution are provided in~\citep{Carvalho2007}.

Conditioned on the truncated HDP-HMM event transition distributions $\{\mb{\phi}_l\}$ and emission parameters $\{\Delta_l\}$, we use a standard backward filtering forward sampling scheme to block sample $Z_{1:T}$.  

\paragraph{AR coefficients, $\{\mb{a}_k\}$}
Each observation $\mb{y}_t$ is generated based on a \emph{matrix} of AR parameters $\mb{A}_{\mb{z}_t} = [\mb{a}_{z_t^{(1)}} \mid \cdots \mid \mb{a}_{z_t^{(N)}}]$.  Thus, sampling $\mb{a}_k$ involves conditioning on $\{\mb{a}_{k'}\}_{k' \neq k}$ and disentangling the contribution of $\mb{a}_k$ on each $\mb{y}_t$. As derived in ~\ref{apndx_sec:HIW_AR_coefs}, the full conditional for $\mb{a}_k$ is a multivariate normal 
\begin{equation}
	p(\mb{a}_k \mid \mb{y}_{1:T,} \mb{z}_{1:T}, Z_{1:T}, \{\mb{a}_{k'}\}_{k' \neq k}, \{\Delta_l\} ) \propto \Ncal(\mb{\mu}_k, \Sigma_k ),
	\label{eq:chStateARCoefsPost}
\end{equation}
where
\begin{align*}
	\Sigma_k^{-1} & = \Sigma_0^{-1} + \sum_{t=1}^T \mb{\bar{Y}}_t^{(\mb{k}^+)} \Delta_{Z_t}^{-1(\mb{k}^+,\mb{k}^+)} \left(\mb{\bar{Y}}_t^{(\mb{k}^+)}\right)^\T, \\
	\Sigma_k^{-1}\mb{\mu}_k & =  \sum_{t=1}^T  \mb{\bar{Y}}_t^{(\mb{k}^+)}\left(\Delta_{Z_t}^{-1(\mb{k}^+,\mb{k}^+)} \mb{y}_t^{(\mb{k}^+)} + \Delta_{Z_t}^{-1(\mb{k}^+,\mb{k}^-)} \mb{\epsilon}_t^{(\mb{k}^-)}\right).
\end{align*}
The vectors $\mb{k}^+$ and $\mb{k}^-$ denote the indices of channels assigned and not assigned to state $k$ at time $t$, respectively. We use these to index into the rows and columns of the vectors $\mb{\epsilon}_t$, $\mb{y}_t$, and matrix $\Delta_{Z_t}$. Each column of matrix $\bar{\mb{Y}}_t^{(\mb{k}^+)}$ is the previous $r$ observations for one of the channels assigned to state $k$ at time $t$.

\paragraph{Hyperparameters}
See \ref{apndx:hypVarsPost} for the prior and full conditionals of the hyperparameters $\gamma_c$, $\kappa_c$, $\alpha_e$, $\kappa_e$, $\gamma_e$, and $\alpha_c = B_0(\Theta)$.

\section{Experiments}

\subsection{Simulation Experiments}
\label{sec:HIW_BPARHMM_simExps}

To initially explore some characteristics of our sparse factorial BP-AR-HMM, we examined a small simulated dataset of six time series in a 2x3 spatial arrangement, with vertices connecting all adjacent nodes (i.e., two cliques of 4 nodes each).  We generated an event of length 2000 time points as follows.  We defined five first-order AR channel states linearly spaced between $-0.9$ and $0.9$ and three event states with covariances shown in the bottom left of Figure~\ref{fig:HIW_BPARHMM_simExp}.  Channel and event state transition distributions were set to $0.99$ and $0.9$, respectively, for a self-transition and uniform between the other states.  Channel feature indicators $f^{(i)}_k$ were simulated from an IBP with $\alpha_c = 10$ (no channel had indicators exceeding the five specified states).  The sampled $\mb{f}^{(i)}$ were then used to modify the channel state transition distributions by setting to 0 transitions to states with $f^{(i)}_k=0$ and then renormalizing.  Using these feature-constrained transition distributions, we simulated sequences $z_{1:T}^{(i)}$ for each channel $i=1,\dots,6$ and for $T=2000$.  The event sequence $Z_{1:T}$ was likewise simulated.  Based on these sampled state sequences, and using the defined state-specific AR coefficients and channel covariances, we generated observations $\mb{y}_{1:T}$ as in Eq.~\eqref{eq:vec_yt_def}.

\begin{figure}[tbp]
	\begin{center}
		\includegraphics[width=\textwidth]{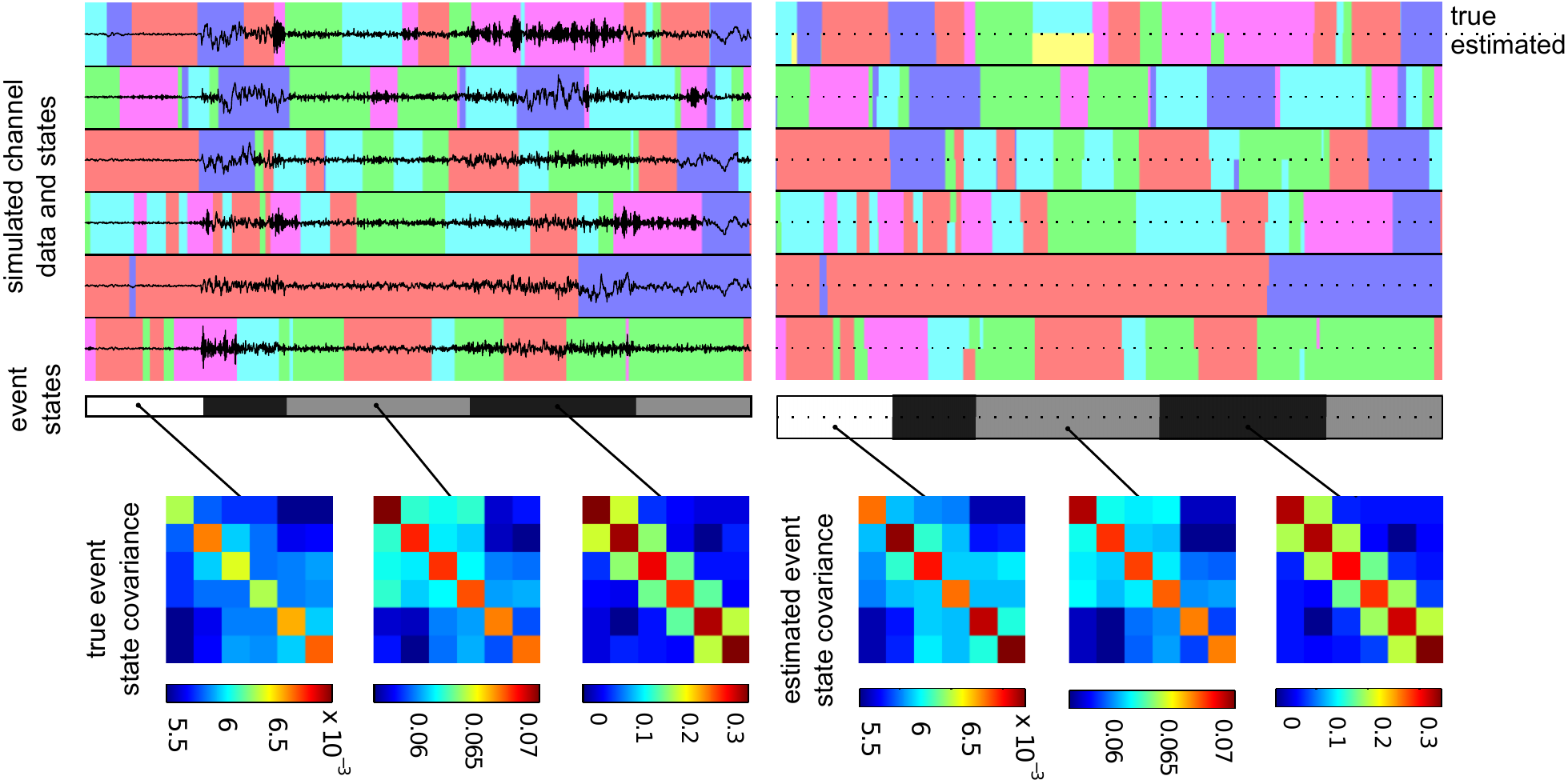}
	\end{center}
	\vspace{-0.1in}\caption[Sparse factorial BP-AR-HMM simulation results]{(\textbf{top left}) The six simulated channel time series overlaid on the five true channel states, denoted by different colors; the three true event states are shown in grayscale in the bar below. (\textbf{top right}) The true and estimated channel (color) and event (grayscale) states shown below for comparison after 6000 MCMC iterations. The true (\textbf{bottom left}) and estimated (\textbf{bottom right}) event state innovation covariances.}
	\label{fig:HIW_BPARHMM_simExp}
\end{figure}

\begin{table}[tbp]
	\centering
	\begin{tabular}{crrr}
		\toprule
		channel state & true $\mb{a}_k$ & post. $\mb{a}_k$ mean & post. $\mb{a}_k$ 95\% interval \\ \midrule
		1 & -0.900 & -0.906 & [-0.917, -0.896] \\
		2 & -0.450 & -0.456 & [-0.474, -0.436] \\
		3 & 0 & -0.009 & [-0.038, 0.020] \\
		4 & 0.450 & 0.445 & [0.425, 0.466] \\
		5 & 0.900 & 0.902 & [0.890, 0.913] \\
		\bottomrule
	\end{tabular}
	\vspace{-0.1in}\caption[Sparse factorial BP-AR-HMM AR parameter posterior estimates.]{The true and estimated values for the channel state coefficients in the simulated dataset. We include the posterior mean and 95\% credible interval.}
	\label{tbl:HIW_BPARHMM_simCoefs}
\end{table}

We ran our MCMC sampler for 6000 iterations, discarding the first 1000 as burn-in and thinning the chain by 10. Figure~\ref{fig:HIW_BPARHMM_simExp} shows the generated data and its true states along with the inferred states and learned channel covariances for a representative posterior sample. The event state matching is almost perfect, and the channel state matching is quite good, though we see that the sampler added an additional (yellow) state in the middle of the first time series when it should have assigned that section to the cyan state. The scale and structure of the estimated event state covariances match the true covariances quite well. Furthermore, Table~\ref{tbl:HIW_BPARHMM_simCoefs} shows how the posterior estimates of the channel state AR coefficients also center well around the true values. 

\subsection{Parsing a Seizure}
\label{sec:szParsing}

We tested the sparse factorial BP-AR-HMM on two similar seizures (events) from a patient of the Children's Hospital of Pennsylvania. These seizures were chosen because qualitatively they displayed a variety of dynamics throughout the beginning, middle, and end of the seizure and thus are ideal for exploring the extent to which our sparse factorial BP-AR-HMM can parse a set of rich neurophysiologic signals. We used the 90 seconds of data after the clinically-determined starts of each seizure from 16 channels, whose spatial layout in the electrode grid is shown in Figure~\ref{fig:szOnset_Parsing} along with the graph encoding our conditional independence assumptions. The data were low-pass filtered and downsampled from 200 to 50 Hz, preserving the clinically important signals but reducing the computational burden of posterior inference. The data was also scaled to have 99\% of values within [-10, 10] for numerical reasons. We examined a $5$th-order sparse factorial BP-AR-HMM and ran 10 MCMC chains for 6000 iterations, discarding 1000 samples as burn-in and using 10-sample thinning. 

The sparse factorial BP-AR-HMM inferred state sequences for the sample corresponding to a minimum expected Hamming distance criterion (\cite{Fox2011}) are shown in Figure~\ref{fig:szOnset_Parsing}. The results were analyzed by a board-certified epileptologist who agreed with the model's judgement in identifying the subtle changes from the background dynamic (cyan) initially present in all channels. The model's grouping of spatially-proximate channels into similar state transition patterns (e.g., channels 03, 07, 11, 15) was clinically intuitive and consistent with his own reading of the raw EEG. Using only the raw EEG, and prior to disclosing our results, he independently identified roughly six points in the duration of the seizure where the dynamics fundamentally change. The three main event state transitions shown in Figure~\ref{fig:szOnset_Parsing} occurred almost exactly at the same time as three of his own marked transitions. The fourth coincides with a major shift in the channel dynamics with most channels transitioning to the green dynamic. The other two transitions he marked that are not displayed occurred after this onset period.  From this analysis, we see that our event states provide an important global summary of the dynamics of the seizure that augments the information conveyed from the channel state sequences.

\begin{figure}[tp]
	\begin{center}
		\includegraphics[width=\textwidth]{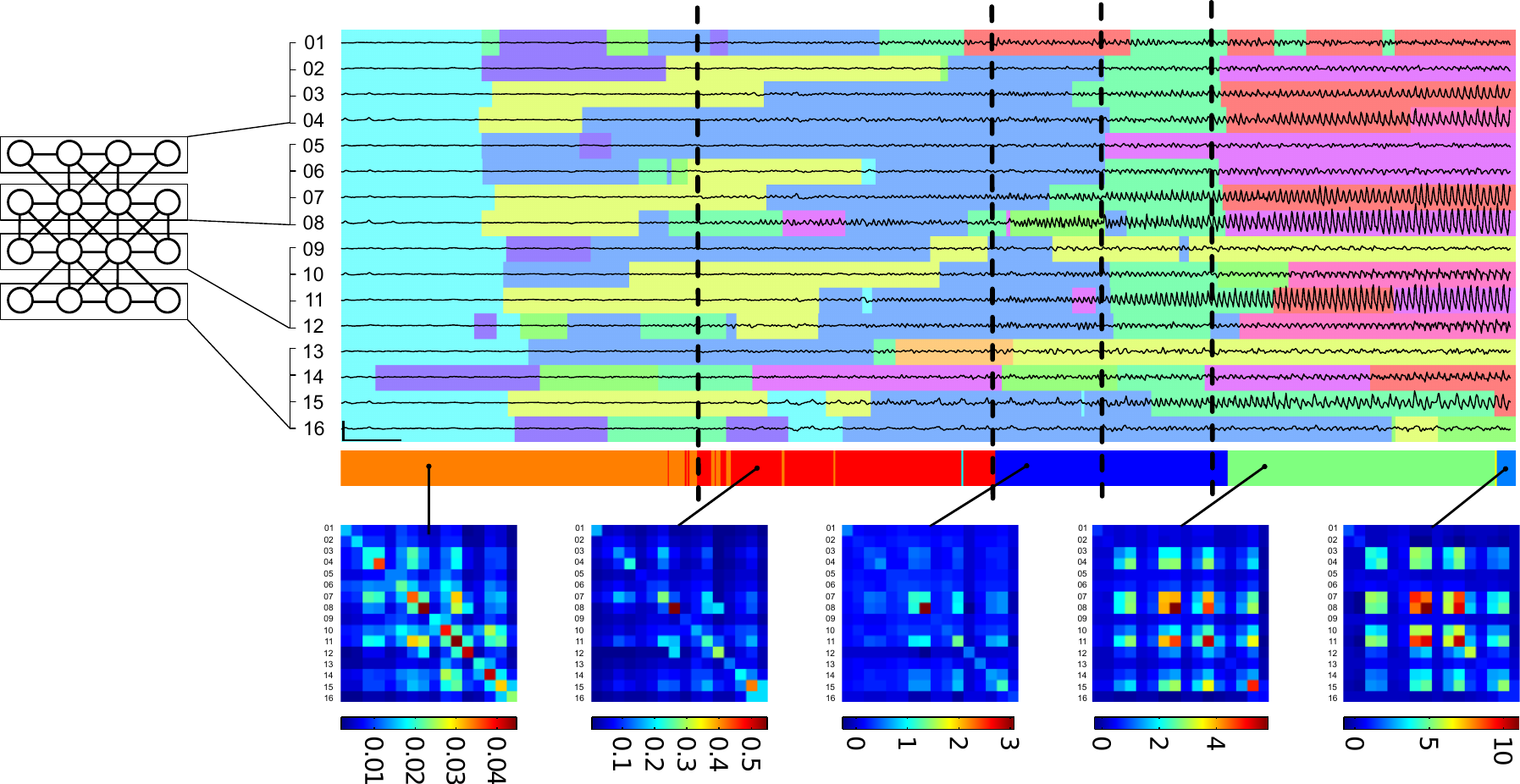}
	\end{center}
	\vspace{-0.1in}\caption[Sparse factorial BP-AR-HMM seizure onset parsing]{The graph used for a 16 channel iEEG electrode and corresponding traces over 25 seconds of a seizure onset with colors indicating the inferred channel states. The event states are shown below along with the associated innovation covariances. Vertical dashed lines indicate the EEG transition times marked independently by an epileptologist. Vertical and horizontal scale bars denote 1 mV and 1 second, respectively. }
	\label{fig:szOnset_Parsing}
\end{figure}

%
%


\paragraph{Clinical relevance}
While interpreting these state sequences and covariances from the model, it is important to keep in mind that they are ultimately estimates of a system whose parsing even highly-trained physicians disagree upon. Nevertheless, we believe that the event states directly describe the activity of particular clinical interest. 

In modeling the correlations between channels, the event states give insight into how different physiologic areas of the brain interact over the course of a seizure. In the clinical workup for resective brain surgery, these event states could help define and specifically quantify the full range of ways in which neurophysiologic regions initiate seizures and how others are recruited over the numerous seizures of a patient. In addition, given fixed model parameters, our model can fit the channel and event state sequences of an hour's worth of 64-channel EEG data in a matter of minutes on a single 8-core machine, possibly facilitating epileptologist EEG annotation of long-term monitoring records.

The ultimate clinical aim of this work, however, involves understanding the relationship between epileptic bursts and seizures. Because the event state aspect of our model involves a Markov assumption, the intrinsic length of the event has little bearing on the states assigned to particular time points. Thus, these event states allow us to straightforwardly compare the neurophysiologic relationship dynamics in short bursts (often less than two seconds long) to those in much longer seizures (on the order of two minutes long), as explored in Section~\ref{sec:subclinical}.  Prior to this analysis, we first examine the importance of our various model components by comparing to baseline alternatives. 

\subsection{Model Comparison}

\paragraph{The advantages of a spatial model}
We explored the extent to which the spatial information and sparse dependencies encoded in the HIW prior improves our predictions of heldout data relative to a number of baseline models. To assess the impact of the sparse dependencies induced by the Gaussian graphical model for $\mb{\epsilon}_t$, we compare to a full-covariance model with an IW prior on $\Delta_l$ (dense factorial). For assessing the importance of spatial correlations, we additionally compare to two alternatives where channels evolve independently: the BP-AR-HMM of~\citet{Fox2009} and an AR-HMM without the feature-based modeling provided by the beta process \cite{Fox2011b}. Both of these models use inverse gamma (IG) priors on the individual channel innovation variances.  We learned a set of AR coefficients $\{\mb{a}_k\}$ and event covariances $\{\Delta_l\}$ on one seizure and then computed the heldout log-likelihood on a separate seizure, constraining it use the learned model parameters from the training seizure.

For the training seizure, MCMC samples were collected over 5000 samples across 10 chains, each with a 1000-sample burn in and 10-sample thinning. To compute the predictive log-likelihood of the heldout seizure, we analytically marginalized the heldout event state sequence $Z_{1:T}$ but perform a Monte Carlo integration over the feature vectors $\mb{f}^{(i)}$ and channel states $\mb{z}_{1:T}$ using our MCMC sampler. For each original MCMC sample generated from the training seizure, a secondary chain is run fixing $\{\mb{a}_k\}$ and $\{\Delta_l\}$ and sampling $z_t^{(i)}$, $Z_t$, $\mb{f}^{(i)}$, $\mb{\eta}^{(i)}$, and $\{\mb{\phi}_l\}$ for the heldout seizure. We approximate $p(\mb{y}_{1:T} \mid \{\mb{\phi}_l\}, \{\mb{a}_k\}, \{\Delta_l\})$ by averaging the secondary chain's closed-form $p(\mb{y}_{1:T} \mid \mb{z}_{1:T}, \{\mb{\phi}_l\}, \{\mb{a}_k\}, \{\Delta_l\})$, described in \ref{apndx:indivVarsPost}.

\begin{figure}[tbp]
	\begin{center}
		\includegraphics[width=\textwidth]{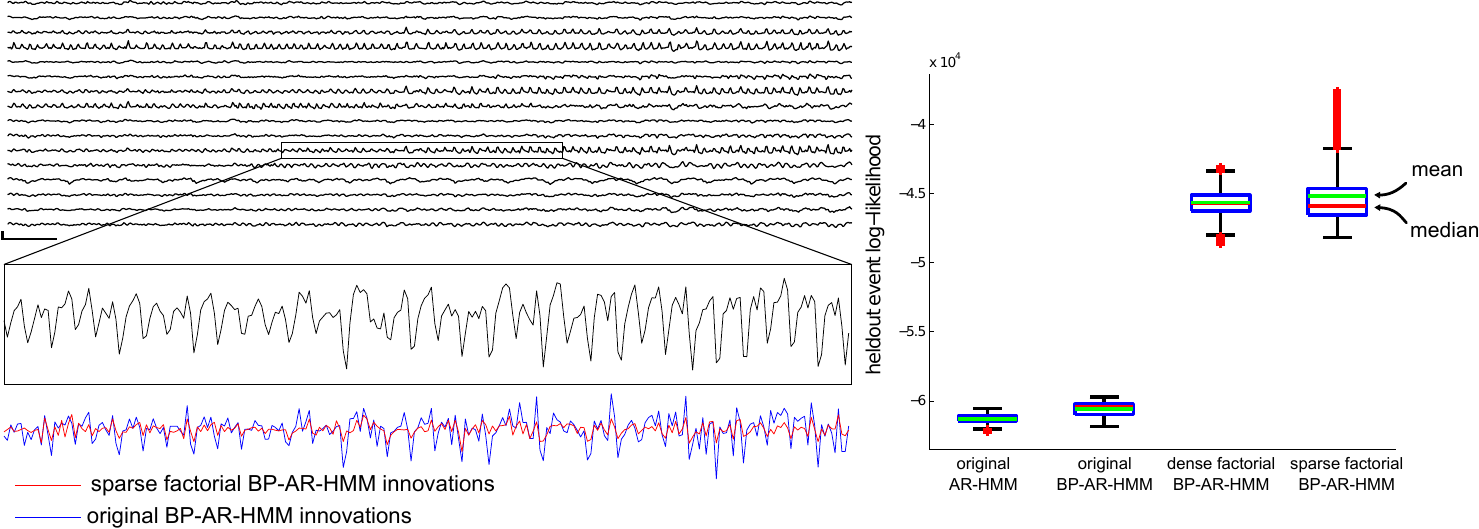}
	\end{center}
	\vspace{-0.1in}\caption[Comparison between factorial and original BP-AR-HMMs.]{(\textbf{left}) An example 16-channel clip of iEEG with the middle section of one channel zoomed in and innovations from the original BP-AR-HMM and our sparse factorial BP-AR-HMM shown below. (\textbf{right}) Boxplots of the heldout event log-likelihoods from the two original and factorial models with mean and median posterior likelihood given in green and red lines. Boxes denote the middle 50\% prediction interval.}
	\label{fig:HIW_BPARHMM_modelComp}
\end{figure}

Figure~\ref{fig:HIW_BPARHMM_modelComp} (left) shows how conditioning on the innovations of neighboring channels in the sparse factorial model improves the prediction of an individual channel, as seen by its reduced innovation trace relative to the original BP-AR-HMM. The quantitative benefits of accounting for these correlations are seen in our predictions of heldout events, as depicted in Figure~\ref{fig:HIW_BPARHMM_modelComp} (right), which compares the heldout log-likelihoods for the original and the factorial models listed above. As expected, the factorial models have significantly larger predictive power than the original models. Though hard to see due to the large factorial/original difference, the BP-based model also improves on the standard non-feature-based AR-HMM. Performance of the sparse factorial model is also at least as competitive as a full-covariance model (dense factorial).  We would expect to see even larger gains for electrode grids with more channels due to the parsimonious representation presented by the graphical model.  Regardless, these results demonstrate that the assumptions of sparsity in the channel dependencies do not adversely affect our performance.

\begin{figure}[tp]
	\begin{center}
		\includegraphics[width=0.5\textwidth]{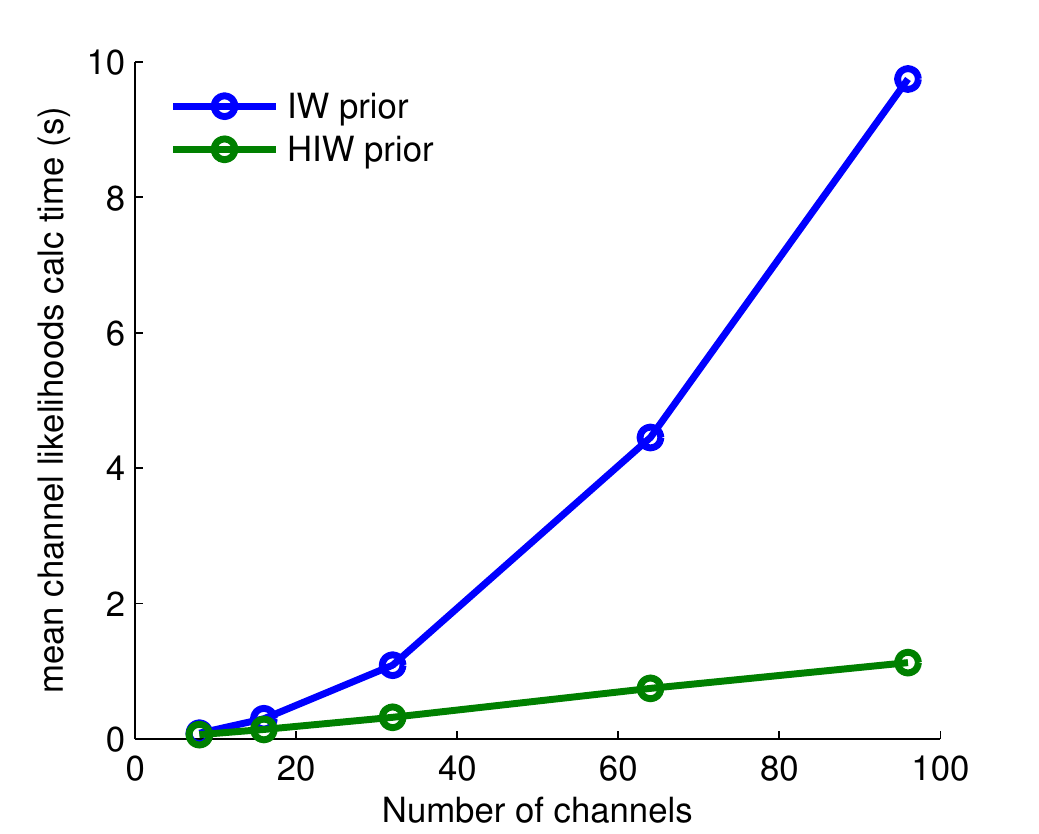}
	\end{center}
	\vspace{-0.1in}\caption[Computational scaling of IW and HIW priors with the number of channels]{The average time per MCMC iteration required to calculate all of the channel likelihoods under each AR model at each time point. }
	\label{fig:chLiksCalcTime}
\end{figure}
\paragraph{The advantages of sparse factorial dependencies}
In addition to providing a parsimonious modeling tool, the sparse dependencies among channels induced by the HIW prior allow our computations to scale linearly to the large number of channels present in iEEG. We compared a dense factorial BP-AR-HMM (entailing a fully-connected spatial graph) and a sparse factorial BP-AR-HMM on five datasets of 8, 16, 32, 64, and 96 channels (from three 32-channel electrodes) from the same seizure used previously. We ran the two models on each of the five datasets for at least 1000 MCMC iterations, using a profiler to tabulate the time spent in each step of the MCMC iteration. 

Figure~\ref{fig:chLiksCalcTime} shows the average time required to calculate the channel likelihoods at each time point under each AR channel state.  This computation is used both for calculating the marginal likelihood (averaging over all the state sequences $z_{1:T}^{(i)}$) required in active feature sampling as well as in sampling the state sequences $z_{1:T}^{(i)}$. In our sparse factorial model, each channel has a constant set of $M$ dependencies, assuming $M$ neighboring channels.  As such, the channel likelihood computation at each time point has an $O(M)$ complexity, implying an $O(MN)$ complexity for calculating the likelihoods of all $N$ channels at each time point. In contrast, the likelihood computation at each time point under the full covariance model had complexity  $O(N)$, implying $O(N^2)$ for calculating all the channel likelihoods.  For $M \ll N$, as is typically the case, our sparse dependency model is significantly more computationally efficient.

Anecdotally, we also found that the IW prior experiments---especially those with larger number of channels---tended to occasionally have numerical underflow problems associated with the inverse term $\Delta^{-1(\mb{i}',\mb{i}')}_{Z_t}$ in the conditional channel likelihood calculation. This underflow in the IW prior model calculations is not surprising since the matrices inverted are of dimension $N-1$ (for $N$ channels), whereas in the HIW prior, the sparse spatial dependencies of the electrode grids make these matrices no larger than eight-by-eight. 

\subsection{Comparing Epileptic Events of Different Scales}
\label{sec:subclinical}
We applied our sparse factorial BP-AR-HMM to six channels of iEEG over 15 events from a human patient with hippocampal depth electrodes. These events comprise 14 short sub-clinical epileptic bursts of roughly five to eight seconds and a final, 2-3 minute clinical seizure. Our hypothesis was that the sub-clinical bursts display initiation dynamics similar to those of a full, clinical seizure and thus contain information about the seizure-generation process.

\begin{figure}[tp]
	\begin{center}
		\includegraphics[width=0.8\textwidth]{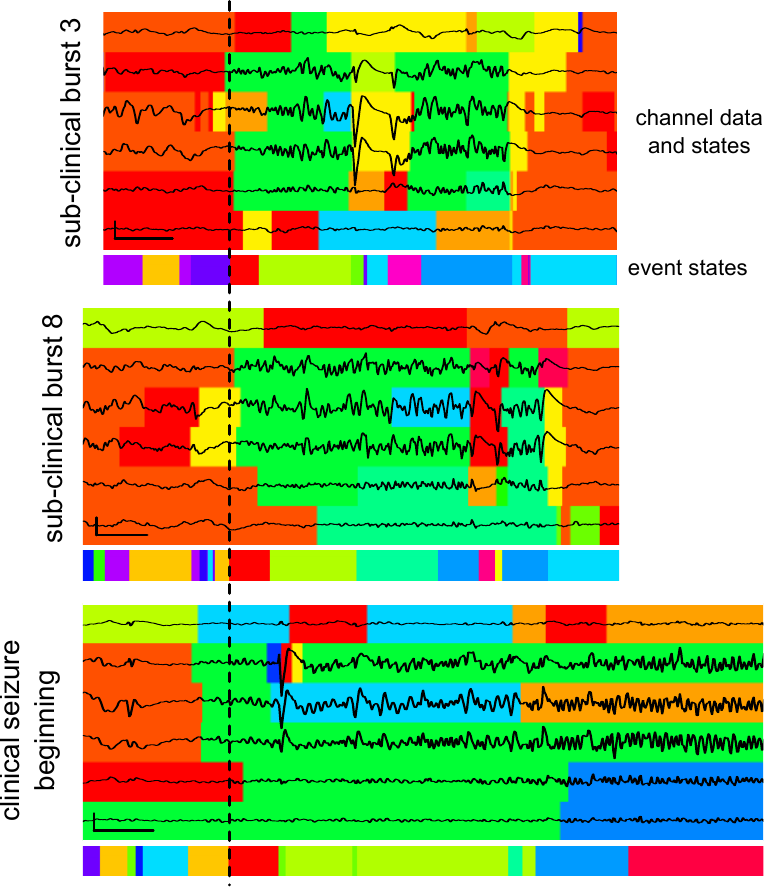}
	\end{center}
	\vspace{-0.1in}\caption[Sparse factorial BP-AR-HMM burst and seizure comparison]{6 iEEG traces from two sub-clinical bursts and onset of the single seizure with colors indicating inferred channel and event states. The dashed lines indicates the start of the red state in the three events. Vertical and horizontal scale bars denote 1mv and 1 second, respectively. } 
	\label{fig:betaBurst_Parsing}
\end{figure}

The events were automatically extracted from the patient's continuous iEEG record by taking sections of iEEG whose median line-length feature \cite{Esteller2001} crossed a preset threshold, also including 10 seconds before and after each event. The iEEG was preprocessed in the same way as in the previous section. The six channels studied came from a depth electrode implanted in the left temporal lobe of the patient's brain. We ran our MCMC sampler jointly on the 15 events.  In particular, the AR channel state and event state parameters, $\{\mb{a}_k\}$ and $\{\Delta_l\}$, were shared between the 15 events such that the parsings of each recording jointly informed the posteriors of these shared parameters.  The hyperparameter settings, number of MCMC iterations, chains, and thinning was as in the experiment of Section~\ref{sec:szParsing}.

Figure~\ref{fig:betaBurst_Parsing} compares two of the 14 sub-clinical bursts and the onset of the single seizure. We have aligned the three events relative to the beginnings of the red event state common to all three, which we treat as the start of the epileptic activity. The individual channel states of the four middle channels are also all green throughout most of the red event state. It is interesting to note that at this time the fifth channel's activity in all three events is much lower than those of the three channels above it, yet it is still assigned to the green state and continues in that state along with the other three channels as the event state switches from the red to the lime green state in all three events. While clinical opinions can vary widely in EEG reading, a physician would most likely not consider this segment of the fifth channel similar to the other three, as our model consistently does. But on a relative voltage axis, the segments actually look quite similar. In a sense, the fifth channel has the same dynamics as the other three but just with smaller magnitude. This kind of relationship is difficult for the human EEG readers to identify and shows how models such as ours are capable of providing new perspectives not readily apparent to a human reader. Additionally, we note the similarities in event state transitions.

The similarities mentioned above, among others, suggest some relationship between these two different classes of epileptic events. However, all bursts make a notable departure from the seizure: a large one-second depolarization in the middle three channels, highlighted at the end by the magenta event state and followed shortly thereafter by the end of the event. Neither the states assigned by our model nor the iEEG itself indicates that dynamic present in the clinical seizure. This difference leads us to posit that perhaps these sub-clinical bursts are a kind of false-start seizure, with similar onset patterns but a disrupting discharge that prevents the event from escalating to a full-blown seizure.

\section{Conclusion}
In this work, we develop a sparse factorial BP-AR-HMM model that allows for shared dynamic regimes between a variable number of time series, asynchronous regime-switching, and an unknown dictionary of dynamic regimes. Key to our model is capturing the between-series correlations and their evolution with a Markov-switching multivariate innovations process.  For scalability, we assume a sparse dependency structure between the time series using a Gaussian graphical model.

This model is inspired by challenges in modeling high-dimensional intracranial EEG time series of seizures, which contain a large variety of small- and large-scale dynamics that are of interest to clinicians. We demonstrate the value of this unsupervised model by showing its ability to parse seizures in a clinically intuitive manner and to produce state of the art out-of-sample predictions on the iEEG data. Finally, we show how our model allows for flexible, large-scale analysis of different classes of epileptic events, opening new, valuable clinical research directions previously too challenging to pursue.  Such analyses have direct relevance to clinical decision-making for patients with epilepsy.



\pagebreak

\bibliographystyle{model1-num-names}
\bibliography{wulsin_dissertation}

\appendix

\section{HMM sum-product algorithm}
\label{apndx:hmmSumProd}

Consider a hidden Markov model of a sequence $y_{1:T}$ with corresponding discrete states $z_{1:T}$, each of which takes one of $K$ values. The joint probability of $y_{1:t}$ and $z_t = \ell$ is
\begin{equation}
	p(y_{1:t}, z_t = \ell) = p(y_t \mid z_t = \ell) \sum_{k=1}^K p(y_{1:t-1}, z_{t-1} = k)p(z_t = \ell \mid z_{t-1} = k),
	\label{eq:HMM_fwdMsgs}
\end{equation}
sometimes called a \emph{forward message}, which depends on a recursive call for $y_{1:t-1}$ and $z_{t-1} = k$ with
\begin{equation}
	p(y_1, z_1 = k) = p(y_1 \mid z_1 = k)p(z_1 = k).
\end{equation}

\begin{algorithm}[tp]
	\caption{HMM forward-filtering algorithm for calculating $p(y_{1:T})$}
 	\label{alg:HMM_margLik}
 	\begin{small}
 	\begin{algorithmic}[1]
		\State let $\mb{\pi}_0$ and $\mb{\pi} = (\mb{\pi}_1 | \cdots | \mb{\pi}_K)^\T$ be the initial and transition distributions 
		\State let $\mb{u}_t \in \mathbb{R}_+^{K}$ be the likelihoods $p(y_t \mid z_t = k)$ for each $k$
		\State let $\mb{\xi}_t \in \mathbb{R}_+^{K}$ be the forward messages $p(y_{1:t-1}, z_{t-1} = k)$ at time $t$ for each $k$
		\State normalize each $\mb{u}_t$ to sum to 1, preventing underflow during computations
		\For {$t = 1,\ldots,T$}
		\State store the marginal log probability over $\mb{u}_t$ in $v_t$ : $v_t \leftarrow \log\left(\mb{1}^\T\mb{u}_t\right)$
		\State normalize $\mb{u}_t$ : $\widetilde{\mb{u}}_t \leftarrow \mb{u}_t / \exp(v_t)$
		\EndFor
		\State calculate and normalize initial forward messages
		\State \quad $\mb{\xi}_1 \leftarrow \widetilde{\mb{u}}_1 \circ \mb{\pi}_0$
		\State \quad $w_1 \leftarrow \mb{1}^\T \mb{\xi}_1$, $\widetilde{\mb{\xi}}_1 \leftarrow \mb{\xi}_1 / w_1$
		\State \quad $v_1 \leftarrow v_1 + \log(w_1)$
		\State propagate messages forward through each time point
		\For {$t = 2,\ldots,T$}
		\State transmit messages forward : $\mb{\xi}_t \leftarrow \widetilde{\mb{u}}_t \circ \left(\mb{\pi}^\T \widetilde{\mb{\xi}}_{t-1} \right)$
		\State normalize : $w_t \leftarrow \mb{1}^\T \mb{\xi}_t$, $\widetilde{\mb{\xi}}_t \leftarrow \mb{\xi}_t / w_t$
		\State $v_t \leftarrow v_t + \log(w_t)$
		\EndFor
		\State calculate final marginal log likelihood : $\log p(y_{1:T}) = \sum_{t=1}^T v_t$
	\end{algorithmic}
	\end{small}
\end{algorithm}

Calculating the marginal likelihood $p(y_{1:T})$ simply involves one last marginalization over $z_T$,
\begin{equation}
	p(y_{1:T}) = \sum_{k=1}^K p(y_{1:T}, z_T = k).
\end{equation}
Algorithm~\ref{alg:HMM_margLik} provides a numerically stable recipe for calculating this marginal likelihood.

\begin{algorithm}[tp]
	\caption{HMM backward-filtering forward-sampling algorithm for block-sampling $z_{1:T}$}
 	\label{alg:HMM_blockSampZs}
 	\begin{small}
 	\begin{algorithmic}[1]
		\State let $\mb{\pi}_0$ and $\mb{\pi} = (\mb{\pi}_1 | \cdots | \mb{\pi}_K)^\T$ be the initial and transition distributions 
		\State let $\mb{u}_t \in \mathbb{R}_+^{K}$ be the likelihoods $p(y_t \mid z_t = k)$ for each $k$
		\State let $\mb{\zeta}_t \in \mathbb{R}_+^{K}$ be the backward messages $p(y_{t+1:T} \mid z_t = k)$ at $t$ for each $k$
		\State normalize each $\mb{u}_t$ to sum to 1, preventing underflow during computations
		\For {$t = 1,\ldots,T$}
		\State $\widetilde{\mb{u}}_t \leftarrow \mb{u}_t / \left(\mb{1}^\T\mb{u}_t\right)$
		\EndFor
		\State calculate backward messages over all time points
		\State $\widetilde{\mb{\zeta}}_T = \mb{1}$
		\For {$t = T-1,\ldots,1$}
		\State transmit messages backward : $\mb{\tau}_{t+1} \leftarrow \widetilde{\mb{u}}_{t+1} \circ \widetilde{\mb{\zeta}}_{t+1}$, $\mb{\zeta}_t \leftarrow \mb{\pi} \mb{\tau}_{t+1}$
		\State normalize : $\widetilde{\mb{\zeta}}_t \leftarrow \mb{\zeta}_t / (\mb{1}^T \mb{\zeta}_t)$
		\EndFor
		\State $\mb{\tau}_{1} \leftarrow \widetilde{\mb{u}}_{1} \circ \widetilde{\mb{\zeta}}_{1}$
		\State sample first time point 
		\State $\mb{q}_1 \leftarrow \mb{\pi}_0 \circ \mb{\tau}_{1}$, $\widetilde{\mb{q}}_1 \leftarrow \mb{q}_1 / (\mb{1}^\T \mb{q})$
		\State $z_1 \sim \widetilde{\mb{q}}_1$
		\State sample other time points
		\For {$t = 2,\ldots,T$}
		\State $\mb{q}_t \leftarrow \mb{\pi}_{z_{t-1}} \circ \mb{\tau}_{t}$, $\widetilde{\mb{q}}_t \leftarrow \mb{q}_t / (\mb{1}^\T\mb{q}_t)$
		\State $z_t \sim \widetilde{\mb{q}}_t$
		\EndFor
	\end{algorithmic}
	\end{small}
\end{algorithm}

We can sample the states $z_{1:T}$ from their joint distribution, also known as block sampling, via a similar recursive formulation. The conditional likelihood of the last $T-t$ samples given the state at $t$ is
\begin{multline}
	p(y_{t+1:T} \mid z_t = \ell) = \\ \sum_{k=1}^K p(y_{t+1} \mid z_{t+1} = k) p(y_{t+2:T} \mid z_{t+1} = k) p(z_{t+1} = k \mid z_{t} = \ell),
	\label{eq:HMM_bkwdMsgs}
\end{multline}
which depends recursively on the \emph{backward messages} $p(y_{t+2:T} \mid z_{t+1} = k)$ for each $k \in \{1,\ldots,K\}$. To sample $z_{1:T}$ at once we use the joint posterior distribution of the entire state sequence $z_{1:T}$, which factors into
\begin{multline}
	p(z_{1:T} \mid y_{1:T}) = \\ p(z_{T} \mid z_{T-1}, y_{1:T}) p(z_{T-1} \mid z_{T-2}, y_{1:T}) \cdots p(z_2 \mid z_1, y_{1:T}) p(z_1 \mid y_{1:T})
\end{multline}
If we first sample $z_1$, we can condition on it to then sample $z_2$ and continue in this fashion until we finish with $z_T$. The posterior for $z_t$ is the product of the backward message, the likelihood of $y_t$ given $z_t$, and the probability of $z_t$ given $z_{t-1}$, 
\begin{equation}
	p(z_t \mid y_{1:T}) \propto p(y_{t+1:T} \mid z_t) p(y_t \mid z_t) p(z_t \mid z_{t-1}),
\end{equation}
where $p(z_t \mid z_{t-1}) \triangleq p(z_1)$ for $t=1$. A numerical stable recipe for this backward-filtering forward-sampling is given in Algorithm~\ref{alg:HMM_blockSampZs}.

\section{Individual channel variables posterior}
\label{apndx:indivVarsPost}

Sampling the variables associated with the individual channel $i$ involves first sampling active features $\mb{f}^{(i)}$ (while marginalizing $z_{1:T}^{(i)}$), then conditioning on these feature assignments $\mb{f}^{(i)}$ to block sample the state sequence $z_{1:T}^{(i)}$, and finally sampling the transition distribution $\mb{\pi}^{(i)}$ given the feature indicators $\mb{f}^{(i)}$ and state sequence $z^{(i)}_{1:T}$. Explicit algorithms for this sampling are given in \citet[Section 4.2.1]{Wulsin2013b}. 

\paragraph{Channel marginal likelihood}
Let $\mb{i}' \subseteq \{1,\ldots,N\}$ index the neighboring channels in the graph upon which channel $i$ is conditioned. The conditional likelihood of observation $y_t^{(i)}$ under AR model $k$ given the neighboring observations $\mb{y}_t^{(\mb{i}')}$ at time $t$ is
\begin{equation}
    p\left(y_t^{(i)} \mid \mb{\widetilde{y}}^{(i)}_t, \mb{y}_t^{(\mb{i}')}, z^{(i)}_t = k, \mb{z}^{(\mb{i}')}_t, Z_t, \{\mb{a}_k\}, \{\Delta_l\} \right) \propto {\cal N}\left(\widetilde{\mu}_t, \widetilde{\sigma}_t^2 \right)
    \label{eq:chanLik_t}
\end{equation}
for
\begin{align}
\begin{aligned}
    \widetilde{\mu}_t &= \mb{a}^T_{k}\mb{\widetilde{y}}^{(i)}_t + \Delta_{Z_t}^{(i,\mb{i}')} \Delta_{Z_t}^{-1(\mb{i}',\mb{i}')} \left(\mb{y}_t^{(\mb{i}')} - \mb{A}_{\mb{z}^{(\mb{i}')}} \mb{\widetilde{Y}}_t^{(\mb{i}')} \right), \\
    \widetilde{\sigma}_t^2 &= \Delta_{Z_t}^{(i,i)} - \Delta_{Z_t}^{(i,\mb{i}')}\Delta_{Z_t}^{-1(\mb{i}',\mb{i}')}\Delta_{Z_t}^{(\mb{i}',i)},
    \label{eq:chanLik_t_muSigSq}
\end{aligned}
\end{align}
which follows from the conditional distribution of the multivariate normal \citep[pg.~579]{Gelman2004}. Using the forward-filtering scheme (see Algorithm~\ref{alg:HMM_margLik}) to marginalize over the exponentially many state sequences $z_{1:T}^{(i)}$, we can calculate the channel marginal likelihood,
\begin{equation}
    p\left(y^{(i)}_{1:T} \mid \mb{y}_{1:T}^{(\mb{i}')}, \mb{z}^{(\mb{i}')}_{1:T}, Z_{1:T}, \mb{f}^{(i)}, \mb{\eta}^{(i)}, \{\mb{a}_k\},\{\Delta_l\}\right), 
    \label{eq:HIW_BPARHMM_chMargLik}
\end{equation}
of channel $i$'s observations over all $t=1,\ldots,T$ given the observations $\mb{y}_{1:T}^{(\mb{i}')}$ and the assigned states $\mb{z}_{1:T}^{(\mb{i}')}$ of neighboring channels $\mb{i}'$ and given the event state sequence $Z_{1:T}$. As previously discussed, taking the non-zero elements of the infinite-dimensional transition distributions $\mb{\pi}^{(i)}$, derived from $\mb{f}^{(i)}$ and $\mb{\eta}^{(i)}$ as in Eq.~\eqref{eq:pi_eta_f}, yields a set of $K^{(i)}$-dimensional active feature transition distributions $\widetilde{\mb{\pi}} = \{\widetilde{\mb{\pi}}_j\}$, reducing this marginalization to a series of matrix-vector products. 

\paragraph{Sampling active features, $\mb{f}^{(i)}$}
We briefly describe the active feature sampling scheme given in detail by \citet{Fox2009}. Recall that for our HIW-spatial BP-AR-HMM, we need to condition on neighboring channel state sequences $\mb{z}^{(\mb{i}')}_{1:T}$ and event state sequences $Z_{1:T}$.  Sampling the feature indicators $\mb{f}^{(i)}$ for channel $i$ via the Indian buffet process (IBP) involves considering those features shared by other channels and those unique to channel $i$. Let $K_+ = \sum_{k=1}^{K} \mb{1}\left(f^{(1)}_k \vee \cdots \vee f^{(N)}_k\right)$ denote the total number of active features used by at least one of the channels. We consider the set of shared features across channels not including those specific to channel $i$ as ${\cal S}^{(-i)} \subseteq \{1,\ldots,K_+\}$ and the set of unique features for channel $i$ as ${\cal U}^{(i)} \subseteq \{1,\ldots,K_+\} / {\cal S}^{(-i)}$.

\subparagraph{Shared features}
The posterior for each shared feature $k \in {\cal S}^{(-i)}$ for channel $i$ is given by
\begin{multline}
    p\left(f^{(i)}_k \mid y^{(i)}_{1:T}, \mb{y}^{(\mb{i}')}_{1:T}, \mb{z}^{(\mb{i}')}_{1:T}, Z_{1:T}, \mb{F}^{-ik}, \mb{f}^{(i)}_{k'\neq k}, \mb{\eta}^{(i)}, \{\mb{a}_k\}, \{\Delta_l\}\right) \propto \\
    p\left(f^{(i)}_k \mid \mb{F}^{-ik}\right)p\left(y^{(i)}_{1:T} \mid \mb{y}^{(\mb{i}')}_{1:T}, \mb{z}^{(\mb{i}')}_{1:T}, Z_{1:T}, \mb{f}^{(i)}_{k'\neq k}, f^{(i)}_k, \mb{\eta}^{(i)}, \{\mb{a}_k\}, \{\Delta_l\} \right),
\end{multline}
where the marginal likelihood of $y^{(i)}_{1:T}$ term (see Eq.~\ref{eq:HIW_BPARHMM_chMargLik}) follows from the sum-product algorithm. Recalling the form of the IBP posterior predictive distribution, we have $p\left(f^{(i)}_k  = 1 \mid \mb{F}^{-ik}\right) = m^{(-i)}_k / N$, where $m^{(-i)}_k$ denotes the number of other channels that use feature $k$. We use this posterior to formulate a Metropolis-Hastings proposal that flips the current indicator value $f^{(i)}_k$ to its complement $\bar{f}^{(i)}_k$ with probability $\rho(\bar{f}^{(i)}_k \mid f^{(i)}_k)$,
\begin{align}
    f^{(i)}_k = \begin{cases} \bar{f}^{(i)}_k, & \mbox{ w.p.}\quad \rho(\bar{f}^{(i)}_k \mid f^{(i)}_k), \\ f^{(i)}_k, & \mbox{otherwise}, \\ \end{cases}
\end{align}
where
\begin{equation}
    \rho(\bar{f}^{(i)}_k \mid f^{(i)}_k) = \min\left(\frac{p\left(\bar{f}^{(i)}_k \mid y^{(i)}_{1:T}, \mb{y}^{(\mb{i}')}_{1:T}, \mb{z}^{(\mb{i}')}_{1:T}, Z_{1:T}, \mb{F}^{-ik}, \mb{f}^{(i)}_{k' \neq k}, \mb{\eta}^{(i)}, \{\mb{a}_k\}, \{\Delta_l\}, \right)}{p\left(f^{(i)}_k \mid y^{(i)}_{1:T}, \mb{y}^{(\mb{i}')}_{1:T}, \mb{z}^{(\mb{i}')}_{1:T}, Z_{1:T}, \mb{F}^{-ik}, \mb{f}^{(i)}_{k' \neq k}, \mb{\eta}^{(i)}, \{\mb{a}_k\}, \{\Delta_l\}, \right)}, 1\right).
    \label{eq:sharedFeatRho}
\end{equation}

\subparagraph{Unique features}
We either propose a new feature or remove a unique feature for channel $i$ using a birth and death reversible jump MCMC sampler \cite{Green1995,Richardson1997,Brooks2011} (see \citet{Fox2011c} for details relevant to the BP-AR-HMM). We denote the number of unique features for channel $i$ as $n^{(i)} = |{\cal U}^{(i)}|$. We define the vector of shared feature indicators as $\mb{f}^{(i)}_{-} = \mb{f}^{(i)}_{k' \mid k' \in {\cal S}^{(-i)}}$ and that for unique feature indicators as $\mb{f}^{(i)}_{+} = \mb{f}^{(i)}_{k' \mid k' \in {\cal U}^{(i)}}$, which together $\left(\mb{f}^{(i)}_{-}, \mb{f}^{(i)}_{+}\right)$ define the full feature indicator vector $\mb{f}^{(i)}$ for channel $i$. Similarly, $\mb{a}^{(i)}_+$ and $\mb{\eta}^{(i)}_+$ describe the model dynamics and transition parameters associated with these unique features. We propose a new unique feature vector ${\mb{f}'}_+$ and corresponding model dynamics ${\mb{a}'}_+$ and transition parameters ${\mb{\eta}'}_+$ (sampled from their priors in the case of feature birth) with a proposal distribution of
\begin{multline}
    q\left(\mb{f}'_{+}, \mb{a}'_+, \mb{\eta}'_+ \mid \mb{f}^{(i)}_{+}, \{\mb{a}_k\}^{(i)}_+, \mb{\eta}^{(i)}_+\right) = \\
    q\left(\mb{f}'_{+} \mid \mb{f}^{(i)}_{+}\right) q\left({\mb{a}'}_+ \mid \mb{f}'_{+}, \mb{f}^{(i)}_{+}, \{\mb{a}_k\}^{(i)}_{+}\right) q\left({\mb{\eta}'}_+ \mid \mb{f}'_{+}, \mb{f}^{(i)}_{+}, \mb{\eta}^{(i)}_{+} \right).
\end{multline}
A new unique feature is proposed with probability $0.5$ and each existing unique feature is removed with probability $0.5/n^{(i)}$. This proposal is accepted with probability 
\begin{multline}
    \rho\left(\mb{f}_{+}', \mb{a}_+', \mb{\eta}_+' \mid \mb{f}^{(i)}_{+}, \{\mb{a}_k\}^{(i)}_+, \mb{\eta}^{(i)}_+\right) = \\
    \min\left(\frac{p\left(y^{(i)}_{1:T} \mid \mb{y}^{(\mb{i}')}_{1:T}, \mb{z}^{(\mb{i}')}_{1:T}, [\mb{f}^{(i)}_-\ \mb{f}'_+ ], \mb{\eta}^{(i)}, \mb{\eta}_+', \{\mb{a}_k\}, \{\Delta_l\} \right)}{p\left(y^{(i)}_{1:T} \mid \mb{y}^{(\mb{i}')}_{1:T}, \mb{z}^{(\mb{i}')}_{1:T}, [\mb{f}^{(i)}_-\ \mb{f}^{(i)}_+ ], \mb{\eta}^{(i)}, \{\mb{a}_k\}, \{\Delta_l\} \right)}\cdot \right. \\
    \left. \frac{\mbox{Poisson}\left(n_i' \mid \alpha_c / N\right)}{\mbox{Poisson}\left(n_i \mid \alpha_c / N\right) } \frac{q\left( \mb{f}^{(i)}_+ \mid \mb{f}'_+ \right)}{ q\left( \mb{f}'_+ \mid \mb{f}^{(i)}_+ \right)},1\right).
    \label{eq:uniqueFeatRho}
\end{multline}

\paragraph{Channel state sequence, $z_{1:T}^{(i)}$}
We block sample the state sequence for all the time points of channel $i$, given that channel's feature-constrained transition distributions $\mb{\pi}^{(i)}$, the state parameters $\{\mb{a}_k\}$, the observations $y^{(i)}_{1:T}$, and the neighboring observations $\mb{y}^{(\mb{i}')}_{1:T}$ and current states $\mb{z}^{(\mb{i}')}_{1:T}$. The joint probability of the state sequence $z^{(i)}_{1:T}$ is given by
\begin{multline}
    p\left(z^{(i)}_{1:T} \mid y^{(i)}_{1:T}, \mb{y}^{(\mb{i}')}_{1:T}, \mb{z}^{(\mb{i}')}_{1:T}, \mb{f}^{(i)}, \mb{\eta}^{(i)},\{\mb{a}_k\},\{\Delta_l\}\right) = \\
    p\left(z^{(i)}_{1} \mid y^{(i)}_{1}, \mb{y}^{(\mb{i}')}_{1}, \mb{z}^{(\mb{i}')}_{1}, \mb{f}^{(i)}, \mb{\eta}^{(i)}, \{\mb{a}_k\},\{\Delta_l\}\right)\cdot\\
    \prod_{t=2}^T p\left(z^{(i)}_{t} \mid y^{(i)}_{t:T}, \mb{y}^{(\mb{i}')}_{t:T}, z^{(i)}_{t-1}, \mb{z}^{(\mb{i}')}_{t:T},\mb{f}^{(i)}, \mb{\eta}^{(i)}, \{\mb{a}_k\},\{\Delta_l\}\right).
\end{multline}
Again following the backward filtering forward sampling scheme (Algorithm~\ref{alg:HMM_blockSampZs}), at each time point $t$ we sample state $z^{(i)}_t$ conditioned on $z^{(i)}_{t-1}$ by marginalizing $z^{(i)}_{t+1:T}$. The conditional probability of $z^{(i)}_t$ is given by
\begin{multline}
    p\left(z^{(i)}_{t} \mid y^{(i)}_{t:T}, \mb{y}^{(\mb{i}')}_{t:T}, z^{(i)}_{t-1}, Z_{1:T}, \mb{z}^{(\mb{i}')}_{t:T},\mb{f}^{(i)}, \mb{\eta}^{(i)},\{\mb{a}_k\},\{\Delta_l\}\right) \propto \\ \mbox{Multi}\left(\mb{\widetilde{\pi}}^{(i)}_{z^{(i)}_{t-1}} \circ \mb{u}^{(i)}_t \circ \mb{\psi}_{t}\right),
\end{multline}
where $\mb{\widetilde{\pi}}^{(i)}_{z^{(i)}_{t-1}}$ is the transition distribution given the assigned state at $t-1$, $\mb{u}^{(i)}_t \in \mathbb{R}^{K^{(i)}}$ is the vector of likelihoods under each possible state at time $t$ (as in Eq.~\eqref{eq:chanLik_t}), and $\mb{\psi}_t \in \mathbb{R}^{K^{(i)}}$ is the vector of backwards messages (see Eq.~\eqref{eq:HMM_bkwdMsgs}) from time point $t+1$ to $t$. 

\paragraph{Channel transition parameters, $\mb{\eta}^{(i)}$}
Following \cite[Supplement]{Hughes2012}, the posterior for the transition variable $\eta^{(i)}_{jk}$ is given by
\begin{equation}
    p(\eta^{(i)}_{jk} \mid z^{(i)}_{1:T}, f^{(i)}_k) \propto \frac{(\eta_{jk}^{(i)})^{n^{(i)}_{jk} + \gamma_c + \kappa_c\delta(j,k) - 1}e^{\eta^{(i)}_{jk}}}{\sum_{k'\mid f^{(i)}_k = 1} \eta^{(i)}_{jk'}},
\end{equation}
where $n^{(i)}_{jk}$ denotes the number of times channel $i$ transitions from state $j$ to state $k$. We can sample from this posterior via two auxiliary variables,
\begin{align}
\begin{aligned}
    \bar{\mb{\eta}}^{(i)}_j & \sim \Dir(\gamma_c + \kappa_c\mb{e}_j + \mb{n}^{(i)}_j) \\
    C_{j}^{(i)} & \sim \GammaD(K\gamma_c + \kappa_c, 1) \\
    \mb{\eta}^{(i)}_{j} & = C^{(i)}_j \bar{\mb{\eta}}^{(i)}_j.
    \label{eq:chEtasPost}
\end{aligned}
\end{align}

\section{Event state variables posterior}
\label{apndx:evVarsPost}

Since we model the event state process with a (truncated approximation to the) HDP-HMM, inference is more straightforward than with the channel states. We block sample the event state sequence $Z_{1:T}$ and then sample the event state transition distributions $\mb{\phi}$. 

\paragraph{Event marginal likelihood}
Let $\mb{z}_t$ denote the vector of $N$ channel states at time $t$. Since the space of $\mb{z}_t$ is exponentially large, we cannot integrate it out to compute the marginal conditional likelihood of the data given the event state sequence $Z_{1:T}$ (and model parameters). Instead, we consider the conditional likelihood of an observation at time $t$ given channel states $\mb{z}_t$ and event state $Z_t = l$,
\begin{equation}
    p(\mb{y}_t \mid \mb{\widetilde{Y}}_t, \mb{z}_t, Z_t = l, \{\mb{a}_k\}, \Delta_l) \propto \Ncal(\mb{A}_{\mb{z}_t}\mb{\widetilde{Y}}_t, \Delta_{l} ).
    \label{eq:evCondLik}
\end{equation}
Recalling Eq.~\eqref{eq:vec_yt_def}, we see that this conditional likelihood of $\mb{y}_t$ is equivalent to a zero-mean multivariate normal model on the channel innovations $\mb{\epsilon}_t$,
\begin{equation*}
    p(\mb{\epsilon}_t \mid Z_t = l, \Delta_l) \propto \Ncal(\mb{0}, \Delta_l ).
\end{equation*}
As with the channel marginal likelihoods, we use the forward-filtering algorithm (see Algorithm~\ref{alg:HMM_margLik}) to marginalize over the possible event state sequences $Z_{1:T}$, yielding a likelihood conditional on the channel states $\mb{z}_t$ and autoregressive parameters $\{\mb{a}_k\}$, in addition to the event transition distribution $\mb{\phi}$ and event state covariances $\{\Delta_l\}$,
\begin{equation}
    p\left(\mb{y}_{1:T} \mid \mb{z}_{1:T}, \mb{\phi}, \{\mb{a}_k\}, \{\Delta_l\}\right).
    \label{eq:evMargLik}
\end{equation}

\paragraph{Event state sequence, $Z_{1:T}$} 
The mechanics of sampling the event state sequence $Z_{1:T}$ directly parallel those of sampling the individual channel state sequences $z_{1:T}^{(i)}$. The joint probability of the event state sequence is given by
\begin{multline}
    p(Z_{1:T} \mid \mb{y}_{1:T}, \mb{z}_{1:T}, \mb{\phi}, \{\mb{a}_k\}, \{\Delta_l\}) \propto \\
    p(Z_1 \mid \mb{y}_1, \mb{z}_1, \mb{\phi}, \{\mb{a}_k\}, \{\Delta_l\}) \prod_{t=2}^T p(Z_t \mid \mb{y}_{t:T},\mb{z}_{t:T}, Z_{t-1}, \mb{\phi}, \{\mb{a}_k\}, \{\Delta_l\}).
\end{multline}
We again use the backward filtering forward sampling scheme of the sum-product algorithm to block sample each event state whose conditional probability distribution over the $L$ states is given by
\begin{equation}
    p\left(Z_t \mid \mb{\widetilde{Y}}_t, \mb{z}_t, \mb{\phi}, \{\mb{a}_k\}, \{\Delta_l\} \right) \propto \Multi\left( \mb{\phi}_{Z_{t-1}} \circ \mb{v}_t \circ \mb{\psi}_t \right),
\end{equation}
where $\mb{\phi}_{Z_{t-1}}$ is the transition distribution given the assigned state at $t-1$, $\mb{v}_t \in \mathbb{R}^L$ is the vector of likelihoods under each of the $L$ possible states at time $t$ (as in Eq.~\eqref{eq:evCondLik}), and $\mb{\psi}_t \in \mathbb{R}^{L}$ is again the vector of backwards messages from time point $t+1$ to $t$. $\circ$ denotes element-wise product.

\paragraph{Event transition parameters, $\mb{\phi}$}
The Dirichlet posterior for the event state $l$'s transition distribution $\mb{\phi}_l$ simply involves the transition counts $\mb{n}_l$ from event state $l$ to all $L$ states,
\begin{equation}
    p\left(\mb{\phi}_l \mid Z_{1:T}, \mb{\beta}\right) \propto \Dir(\alpha_e\mb{\beta} + \mb{\rm e}_l \kappa_e + \mb{n}_l),
    \label{eq:evStateTransParamsSamp}
\end{equation}
for global weights $\mb{\beta}$, concentration parameter $\alpha_e$, and self-transition parameter $\kappa_e$. 

\paragraph{Global transition parameters, $\mb{\beta}$}
The Dirichlet posterior of the global transition distribution $\mb{\beta}$ involves the auxiliary variables $(\bar{m}_{\cdot 1}, \ldots,\bar{m}_{\cdot L})$,
\begin{equation}
    p(\mb{\beta} \mid Z_{1:T}) \propto \Dir(\gamma_e/L + \bar{m}_{\cdot 1}, \ldots, \gamma_e/L + \bar{m}_{\cdot L}),
\end{equation}
where these auxiliary variables are defined as
\begin{align}
\begin{aligned}
    \bar{m}_{ll'} & = \begin{cases} m_{ll'}, & l \neq l' \\ m_{ll} - w_{l}, & l = l' \end{cases} \\
        m_{ll'} &= \sum_{r=1}^{n_{ll'}} \theta_r \\
    \theta_r & \sim \Ber\left( \frac{\alpha_e\beta_l + \kappa_e \delta(l,l')}{ \alpha_e\beta_l + \delta(l,l') + r} \right), \\
    w_l & \sim \Bin\left(m_{ll'}, \frac{\rho_e}{\rho_e + \beta_l(1-\rho_e)}\right)
    \label{eq:chap4_betaAuxVars}
\end{aligned}
\end{align}
and $m_{\cdot l'} = \sum_l m_{ll'}$. See \citet[Appendix A]{Fox2009a} for full derivations.  

\section{Hyperparameters posterior}
\label{apndx:hypVarsPost}

Below we give brief descriptions for the MCMC sampling of the hyperparameters in our model. Full derivations are given in \citet[Section 5.2.3, Appendix C]{Fox2009a}. 

\paragraph{Channel dynamics model hyperparameters, $\gamma_c$, $\kappa_c$}
We use Metropolis-Hastings steps to propose a new value $\gamma_c'$ from gamma distributions with fixed variance $\sigma^2_{\gamma_c}$  and accept with probability $\min(r(\gamma_c' \mid \gamma_c),1)$,
\begin{align}
    r(\gamma_c' \mid \gamma_c) & = \frac{p(\{\mb{\pi}^{(i)}\} \mid \gamma_c', \kappa, \{\mb{f}^{(i)}\})p(\gamma_c' \mid \gamma_c^2 / \sigma^2_{\gamma_c}, \gamma_c / \sigma^2_{\gamma_c})p(\gamma_c \mid \gamma_c', \sigma^2_{\gamma_c}) }{p(\{\mb{\pi}^{(i)}\} \mid \gamma_c, \kappa, \{\mb{f}^{(i)}\})q(\gamma_c \mid \gamma_c^2 / \sigma^2_{\gamma_c}, \gamma_c / \sigma^2_{\gamma_c})q(\gamma_c' \mid \gamma_c, \sigma^2_{\gamma_c})} \nonumber \\
    & = \frac{ p(\{\mb{\pi}^{(i)}\} \mid \gamma_c', \kappa, \{\mb{f}^{(i)}\})}{p(\{\mb{\pi}^{(i)}\} \mid \gamma_c, \kappa, \{\mb{f}^{(i)}\})} \frac{\Gamma(\nu) \gamma_c^{\nu' - \nu - a} }{ \Gamma(\nu') \gamma_c^{\nu - \nu' - a}} \exp\left( -b(\gamma_c' - \gamma_c)\sigma^{2(\nu-\nu')}_{\gamma_c} \right),
    \label{eq:gamma_c_PostSamp}
\end{align}
where $\nu = \gamma_c^2/\sigma^2_{\gamma_c}$, $\nu' = {\gamma_c'}^2/\sigma^2_{\gamma_c}$, and we have a $\GammaD(a,b)$ prior on $\gamma_c$. The likelihood term $p(\{\mb{\pi}^{(i)}\} \mid \gamma_c', \kappa, \{\mb{f}^{(i)}\})$ follows from the Dirichlet distribution and is given by
\begin{multline}
  p(\{\mb{\pi}^{(i)}\} \mid \gamma_c', \kappa, \{\mb{f}^{(i)}\}) = \\ \prod_{i=1}^N \prod_{k=i}^{K^{(i)}} \left( \frac{\Gamma(\gamma_c' K^{(i)} + \kappa_c)}{\left( \prod_{k'}^{K^{(i)}-1}\Gamma(\gamma_c')\right)\Gamma(\gamma_c'+\kappa_c)} \prod_{j=1}^{K^{(i)}} \left(\widetilde{\pi}_{kj}^{(i)}\right)^{\gamma_c' + \kappa_c\delta(k,j)-1} \right),
\end{multline}
where that for $p(\{\mb{\pi}^{(i)}\} \mid \gamma_c, \kappa, \{\mb{f}^{(i)}\})$ is similar. Recall that the transition parameters $\{\mb{\pi}^{(i)}\}$ are independent over $i$, and thus their likelihoods multiply. The proposal and acceptance ratio for $\kappa_c$ is similar.

\paragraph{Channel active features model hyperparameter $\alpha_c$}
We place a $\mbox{Gamma}(a_{\alpha_c},b_{\alpha_c})$ prior on $\alpha_c$, which implies a gamma posterior of the form
\begin{equation}
    p(\alpha_c \mid \{\mb{f}^{(i)}\}) \propto \GammaD(a_{\alpha_c} + K_+, b_{\alpha_c} + \sum_{i=1}^N (1/i)),
    \label{eq:alpha_c_PostSamp}
\end{equation}
where $K_+ = \sum_{k=1}^{K} \mb{1}\left(f^{(1)}_k \vee \cdots \vee f^{(N)}_k\right)$ denotes the number of channel states activated in at least one of the channels.

\paragraph{Event dynamics model hyperparameters, $\gamma_e$, $\alpha_e$, $\kappa_e$, $\rho_e$}
Instead of sampling $\alpha_e$ and $\kappa_e$ independently, we an additional parameter $\rho_e = \kappa_e/(\alpha_e + \kappa_e)$ and sample $(\alpha_e + \kappa_e)$ and $\rho_e$, which is simpler than sampling $\alpha_e$ and $\kappa_e$ independently.

\subparagraph{$\mb{(\alpha_e + \kappa_e)}$}
With a $\GammaD(a_{\alpha_e+\kappa_e}, b_{\alpha_e+\kappa_e})$ prior on $(\alpha_e + \kappa_e)$, we use auxiliary variables $\{r_l\}_{l=1}^L$ and $\{s_l\}_{l=1}^L$ to define the posterior,
\begin{equation}
    p\left( \alpha_e + \kappa_e \mid Z_{1:T} \right) \propto \GammaD\left(a_{\alpha_e+\kappa_e} + \bar{m}_{\cdot\cdot} - \sum_{l=1}^L s_l, b_{\alpha_e+\kappa_e} - \sum_{l=1}^{L} \log r_l\right),
    \label{eq:alphaKappa_e_PostSamp}
\end{equation}
where $\bar{m}_{\cdot\cdot} = \sum_{l,l'=1}^L \bar{m}_{ll'}$ is the sum over auxiliary variables $\bar{m}_{ll'}$ defined in Eq.~\eqref{eq:chap4_betaAuxVars}, and the auxiliary variables $r_l$ and $s_l$ are sampled as
\begin{align*}
    r_l &\sim \mbox{Beta}(\alpha+\kappa+1,n_{l\cdot}), \\
    s_l & \sim \mbox{Ber}(n_{l\cdot} / (n_{l\cdot} + \alpha + \kappa)).
\end{align*}

\subparagraph{$\mb{\rho_e}$}
With a $\BetaD(c_{\rho_e},d_{\rho_e})$ prior on $\rho_e$, we use auxiliary variables $\{w_{l\cdot}\}_{l=1}^L$ to define the posterior,
\begin{equation}
    p(\rho_e \mid \bar{\mb{m}}, \mb{\beta}) \propto \mbox{Beta}\left( c_{\rho_e} + \sum_{l=1}^L w_{l\cdot}, d_{\rho_e} + \bar{m}_{\cdot\cdot} - \sum_{l=1}^L w_{l\cdot}\right).
    \label{eq:rho_e_PostSamp}
\end{equation}
For $w_{ls} \sim \Ber(\rho)$ over $s=1,\ldots,m_{ll}$, the posterior of the auxiliary variable $w_{l\cdot}$ is
\begin{equation} 
    p(w_{l\cdot} \mid \bar{m}_{ll}, \beta_l) \propto \Bin(\bar{m}_{ll}, \rho_e + \beta_l(1-\rho_e))
\end{equation}

\subparagraph{$\mb{\gamma_e}$}
With a $\GammaD(a_{\gamma_e}, b_{\gamma_e})$ prior on $\gamma_e$, we use auxiliary variables $v$ and $q$ to define the posterior,
\begin{equation}
    p(\gamma_e \mid \mb{m} ) \propto \GammaD\left(a_{\gamma_e} - q + \sum_{l=1}^{L} \mb{1}(\bar{m}_{\cdot l} > 0), b_{\gamma_e} - \log v \right).
    \label{eq:gamma_e_PostSamp}
\end{equation}
The auxiliary variables are sampled via
\begin{align*}
    v &\sim \BetaD(\gamma_e + 1, \bar{m}_{\cdot\cdot}), \\
    q &\sim \Ber(\bar{m}_{\cdot\cdot} / (\gamma_e + \bar{m}_{\cdot\cdot})).
\end{align*}

\section{Autoregressive state coefficients posterior}
\label{apndx_sec:HIW_AR_coefs}

Recall that our prior on the autoregressive coefficients $\mb{a}_k$ is a multivariate normal with zero mean and covariance $\Sigma_0$, 
\begin{align}
	p(\mb{a}_k \mid \Sigma_0) & \propto \Ncal(\mb{0}, \Sigma_0) \nonumber \\
	\log p(\mb{a}_k \mid \Sigma_0) & \propto -\frac{1}{2}\mb{a}_k^\T \Sigma_0^{-1} \mb{a}_k.
\end{align}
The conditional event likelihood given the channel states $\mb{z}_{1:T}$ and the event states $Z_{1:T}$ is 
\begin{align}
	p(\mb{y}_{1:T} \mid \mb{z}_{1:T}, Z_{1:T}, \{\mb{a}_k\}, \{\Delta_l\}) & \propto \prod_{t=1}^T \Ncal(\mb{y}_t; \mb{A}_{\mb{z}_t}\mb{\widetilde{Y}}, \Delta_{Z_t}) \nonumber \\
	\log p(\mb{y}_{1:T} \mid \mb{z}_{1:T}, Z_{1:T}, \{\mb{a}_k\}, \{\Delta_l\}) & \propto -\frac{1}{2}\sum_{t=1}^T (\mb{y}_t - \mb{A}_{\mb{z}_t}\mb{\widetilde{Y}}_t)^\T \Delta_{Z_t}^{-1} (\mb{y}_t - \mb{A}_{\mb{z}_t}\mb{\widetilde{Y}}_t).
\end{align}
The product of these prior and likelihood terms is the joint distribution over $\mb{a}_k$ and $\mb{y}_{1:T}$,
\begin{multline} \label{eqn:ayJointDist}
	\log p(\mb{a}_k, \mb{y}_{1:T} \mid \mb{z}_{1:T}, Z_{1:T}, \{\mb{a}_{k'}\}_{k' \neq k}, \{\Delta_l\} ) \propto \\
	-\frac{1}{2}\mb{a}_k^\T \Sigma_0^{-1} \mb{a}_k-\frac{1}{2}\sum_{t=1}^T (\mb{y}_t - \mb{A}_{\mb{z}_t}\mb{\widetilde{Y}}_t)^\T \Delta_{Z_t}^{-1} (\mb{y}_t - \mb{A}_{\mb{z}_t}\mb{\widetilde{Y}}_t).
\end{multline}
We take a brief tangent to prove a useful identity,
\begin{lemma}
	Let the column vector $\mb{x} \in \mathbb{R}^m$ and the symmetric matrix $A \in \mathbb{S}^{m\times m}$ be defined as
	\[ \mb{x} = \mat{c}{\mb{y} \\ \mb{z}} \quad\mbox{ and }\quad A = \mat{cc}{B & C \\ C^\T & D}, \]
	where $\mb{y} \in \mathbb{R}^p$, $\mb{z} \in \mathbb{R}^q$, $B \in \mathbb{S}^{p\times p}$, $D \in \mathbb{S}^{q\times q}$, $C \in \mathbb{R}^{p\times q}$, and $m = p + q$. Then 
\begin{equation}
	\mb{x}^\T A \mb{x} = \mb{y}^\T B \mb{y} + \mb{z}^\T D \mb{z} + 2\mb{y}^\T C \mb{z}.
\end{equation}
\end{lemma}
\begin{proof}
	\begin{align*}
		\mb{x}^\T A \mb{x} &= \mat{cc}{\mb{y}^\T &\mb{z}^\T} \mat{cc}{B & C \\ C^\T & D} \mat{c}{\mb{y} \\ \mb{z}} \\
		& = \mat{cc}{\mb{y}^\T &\mb{z}^\T} \mat{c}{B\mb{y} + C\mb{z} \\ C^\T\mb{y} + D\mb{z} } \\
		& = \mb{y}^\T B \mb{y} + \mb{y}^\T C \mb{z} + \mb{z}^\T C^\T \mb{y} + \mb{z}^\T D \mb{z} \\
		& = \mb{y}^\T B \mb{y} + \mb{z}^\T D \mb{z} + 2\mb{y}^\T C \mb{z}
	\end{align*}
\end{proof}
Note that this identity also holds for any permutation $p$ applied to the rows of $\mb{x}$ and the rows and columns of $A$. We now can manipulate the likelihood term of Eq.~\eqref{eqn:ayJointDist} into a form that separates $\mb{a}_k$ from $\mb{a}_{k'\neq k}$. Suppose that $\mb{k}^+$ denotes the indices of the $N$ channels where $z_t^{(i)} = k$ and $\mb{k}^- = \{1,\ldots,N\} / \mb{k}^+$ denotes those for whom $z_t^{(i)} \neq k$. Furthermore, we use the superscript indexing on these two sets of indices to select the corresponding portions of the $\mb{y}_t$ vector and the $\mb{A}_{\mb{z}_t}$, $\mb{\widetilde{Y}}_T$, and $\Delta_{Z_t}^{-1}$ matrices. We start by decomposing the likelihood term into three parts,
\begin{equation}
\begin{array}{l}
	(\mb{y}_t - \mb{A}_{\mb{z}_t}\mb{\widetilde{Y}}_t)^\T \Delta_{Z_t}^{-1} (\mb{y}_t - \mb{A}_{\mb{z}_t}\mb{\widetilde{Y}}_t) \propto \\
	\qquad \left(\mb{y}_t^{(\mb{k}^+)} - \mb{A}_{\mb{z}_t}^{(\mb{k}^+,\mb{k}^+)}\mb{\widetilde{Y}}_t^{(\mb{k}^+,\mb{k}^+)}\right)^\T \Delta_{Z_t}^{-1(\mb{k}^+,\mb{k}^+)} \left(\mb{y}_t^{(\mb{k}^+)} - \mb{A}_{\mb{z}_t}^{(\mb{k}^+,\mb{k}^+)}\mb{\widetilde{Y}}_t^{(\mb{k}^+,\mb{k}^+)}\right) + \\
	\qquad 2\left(\mb{y}_t^{(\mb{k}^+)} - \mb{A}_{\mb{z}_t}^{(\mb{k}^+,\mb{k}^+)}\mb{\widetilde{Y}}_t^{(\mb{k}^+,\mb{k}^+)}\right)^\T \Delta_{Z_t}^{-1(\mb{k}^+,\mb{k}^-)} \left(\mb{y}_t^{(\mb{k}^-)} - \mb{A}_{\mb{z}_t}^{(\mb{k}^-,\mb{k}^-)}\mb{\widetilde{Y}}_t^{(\mb{k}^-,\mb{k}^-)}\right) + \\
	\qquad \left(\mb{y}_t^{(\mb{k}^-)} - \mb{A}_{\mb{z}_t}^{(\mb{k}^-,\mb{k}^-)}\mb{\widetilde{Y}}_t^{(\mb{k}^-,\mb{k}^-)}\right)^\T \Delta_{Z_t}^{-1(\mb{k}^-,\mb{k}^-)} \left(\mb{y}_t^{(\mb{k}^-)} - \mb{A}_{\mb{z}_t}^{(\mb{k}^-,\mb{k}^-)}\mb{\widetilde{Y}}_t^{(\mb{k}^-,\mb{k}^-)}\right),
\end{array}
\end{equation}
which we then insert into our previous expression for the joint distribution of $\mb{a}_k$ and $\mb{y}_{1:T}$ (Eq.~\eqref{eqn:ayJointDist}),
\begin{equation}
\begin{array}{l}
	\log p(\mb{a}_k, \mb{y}_{1:T} \mid \mb{z}_{1:T}, Z_{1:T}, \{\mb{a}_{k'}\}_{k' \neq k}, \{\Delta_l\} ) \propto -\frac{1}{2}\mb{a}_k^\T \Sigma_0^{-1} \mb{a}_k - \\
	\qquad \frac{1}{2}\sum_{t=1}^T \left\{ \left(\mb{y}_t^{(\mb{k}^+)} - \mb{A}_{\mb{z}_t}^{(\mb{k}^+,\mb{k}^+)}\mb{\widetilde{Y}}_t^{(\mb{k}^+,\mb{k}^+)}\right)^\T \Delta_{Z_t}^{-1(\mb{k}^+,\mb{k}^+)} \left(\mb{y}_t^{(\mb{k}^+)} - \mb{A}_{\mb{z}_t}^{(\mb{k}^+,\mb{k}^+)}\mb{\widetilde{Y}}_t^{(\mb{k}^+,\mb{k}^+)}\right) + \right. \\
	\qquad 2\left(\mb{y}_t^{(\mb{k}^+)} - \mb{A}_{\mb{z}_t}^{(\mb{k}^+,\mb{k}^+)}\mb{\widetilde{Y}}_t^{(\mb{k}^+,\mb{k}^+)}\right)^\T \Delta_{Z_t}^{-1(\mb{k}^+,\mb{k}^-)} \left(\mb{y}_t^{(\mb{k}^-)} - \mb{A}_{\mb{z}_t}^{(\mb{k}^-,\mb{k}^-)}\mb{\widetilde{Y}}_t^{(\mb{k}^-,\mb{k}^-)}\right) + \\
	\qquad \left.\left(\mb{y}_t^{(\mb{k}^-)} - \mb{A}_{\mb{z}_t}^{(\mb{k}^-,\mb{k}^-)}\mb{\widetilde{Y}}_t^{(\mb{k}^-,\mb{k}^-)}\right)^\T \Delta_{Z_t}^{-1(\mb{k}^-,\mb{k}^-)} \left(\mb{y}_t^{(\mb{k}^-)} - \mb{A}_{\mb{z}_t}^{(\mb{k}^-,\mb{k}^-)}\mb{\widetilde{Y}}_t^{(\mb{k}^-,\mb{k}^-)}\right) \right\}. \\
\end{array}
\end{equation}
Conditioning on $\mb{y}_{1:T}$ allows us to absorb the third term of the sum into the proportionality, and after replacing $\mb{A}_{\mb{z}_t}^{(\mb{k}^+,\mb{k}^+)}\mb{\widetilde{Y}}^{(\mb{k}^+,\mb{k}^+)}$ with a more explicit expression, we have
\begin{equation}
\begin{array}{l}
	\log p(\mb{a}_k \mid \mb{y}_{1:T}, \mb{z}_{1:T}, Z_{1:T}, \{\mb{a}_{k'}\}_{k' \neq k}, \{\Delta_l\} ) \propto -\frac{1}{2}\mb{a}_k^\T \Sigma_0^{-1} \mb{a}_k - \\
	\qquad \frac{1}{2}\sum_{t=1}^\T \left\{ \left(\mb{y}_t^{(\mb{k}^+)} - \left[\mb{\widetilde{y}}_t^{(k_1^+)} \mid \cdots \mid \mb{\widetilde{y}}_t^{(k_{|\mb{k}^+|}^+)}\right]^\T\mb{a}_k \right)^\T \left( \Delta_{Z_t}^{-1(\mb{k}^+,\mb{k}^+)} \right) \cdot\right.\\
	\qquad \left(\mb{y}_t^{(\mb{k}^+)} - \left[\mb{\widetilde{y}}_t^{(k_1^+)} \mid \cdots \mid \mb{\widetilde{y}}_t^{(k_{|\mb{k}^+|}^+)}\right]^\T\mb{a}_k \right) + \\
	\qquad 2\left(\mb{y}_t^{(\mb{k}^+)} - \left[\mb{\widetilde{y}}_t^{(k_1^+)} \mid \cdots \mid \mb{\widetilde{y}}_t^{(k_{|\mb{k}^+|}^+)}\right]^\T\mb{a}_k \right)^\T \left( \Delta_{Z_t}^{-1(\mb{k}^+,\mb{k}^-)}\right) \cdot \\
	\qquad \left. \left(\mb{y}_t^{(\mb{k}^-)} - \mb{A}_{\mb{z}_t}^{(\mb{k}^-,\mb{k}^-)}\mb{\widetilde{Y}}_t^{(\mb{k}^-,\mb{k}^-)}\right) \right\},
\end{array}
\end{equation}
which we can further expand to yield
\begin{equation}
\begin{array}{l}
	\log p(\mb{a}_k \mid \mb{y}_{1:T,} \mb{z}_{1:T}, Z_{1:T}, \{\mb{a}_{k'}\}_{k' \neq k}, \{\Delta_l\} ) \propto -\frac{1}{2}\mb{a}_k^\T \Sigma_0^{-1} \mb{a}_k - \\
	\qquad \frac{1}{2}\sum_{t=1}^T \left\{ \left(\mb{y}_t^{(\mb{k}^+)}\right)^\T \left(\Delta_{Z_t}^{-1(\mb{k}^+,\mb{k}^+)}\right) \left(\mb{y}_t^{(\mb{k}^+)}\right) + \right. \\
	\qquad \left(\mb{a}_k^\T \left[\mb{\widetilde{y}}_t^{(k_1^+)} \mid \cdots \mid \mb{\widetilde{y}}_t^{(k_{|\mb{k}^+|}^+)}\right]\right)\left(\Delta_{Z_t}^{-1(\mb{k}^+,\mb{k}^+)}\right) \left(\left[\mb{\widetilde{y}}_t^{(k_1^+)} \mid \cdots \mid \mb{\widetilde{y}}_t^{(k_{|\mb{k}^+|}^+)}\right]^\T\mb{a}_k\right) - \\
	\qquad \left. 2\left(\mb{y}_t^{(\mb{k}^+)}\right)^\T \left(\Delta_{Z_t}^{-1(\mb{k}^+,\mb{k}^+)}\right) \left(\left[\mb{\widetilde{y}}_t^{(k_1^+)} \mid \cdots \mid \mb{\widetilde{y}}_t^{(k_{|\mb{k}^+|}^+)}\right]^\T\mb{a}_k\right) \right\} - \\
	\qquad \sum_{t=1}^T \left\{ \mb{y}_t^{\T(\mb{k}^+)}\Delta_{Z_t}^{-1(\mb{k}^+,\mb{k}^-)}\left( \mb{y}_t^{(\mb{k}^-)} - \mb{A}_{\mb{z}_t}^{(\mb{k}^-,\mb{k}^-)}\mb{\widetilde{Y}}_t^{(\mb{k}^-,\mb{k}^-)}\right) - \right. \\
	\qquad \left. \left(\mb{a}_k^\T\left[\mb{\widetilde{y}}_t^{(k_1^+)} \mid \cdots \mid \mb{\widetilde{y}}_t^{(k_{|\mb{k}^+|}^+)}\right]\right)\left(\Delta_{Z_t}^{-1(\mb{k}^+,\mb{k}^-)}\right) \left(\mb{y}_t^{(\mb{k}^-)} - \mb{A}_{\mb{z}_t}^{(\mb{k}^-,\mb{k}^-)}\mb{\widetilde{Y}}_t^{(\mb{k}^-,\mb{k}^-)}\right) \right\}.
\end{array}
\end{equation}
Absorbing more terms unrelated to $\mb{a}_k$ into the proportionality, we have
\begin{equation}
\begin{array}{l}
	\log p(\mb{a}_k \mid \mb{y}_{1:T,} \mb{z}_{1:T}, Z_{1:T}, \{\mb{a}_{k'}\}_{k' \neq k}, \{\Delta_l\} ) \propto -\frac{1}{2}\mb{a}_k^\T \Sigma_0^{-1} \mb{a}_k - \\
	\qquad \frac{1}{2}\sum_{t=1}^T \left\{ \left(\mb{a}_k^\T \left[\mb{\widetilde{y}}_t^{(k_1^+)} \mid \cdots \mid \mb{\widetilde{y}}_t^{(k_{|\mb{k}^+|}^+)}\right]\right) \left(\Delta_{Z_t}^{-1(\mb{k}^+,\mb{k}^+)}\right) \left(\left[\mb{\widetilde{y}}_t^{(k_1^+)} \mid \cdots \mid \mb{\widetilde{y}}_t^{(k_{|\mb{k}^+|}^+)}\right]^\T\mb{a}_k\right) - \right. \\
	\qquad \left. 2\left(\mb{y}_t^{(\mb{k}^+)}\right)^\T \left(\Delta_{Z_t}^{-1(\mb{k}^+,\mb{k}^+)}\right) \left(\left[\mb{\widetilde{y}}_t^{(k_1^+)} \mid \cdots \mid \mb{\widetilde{y}}_t^{(k_{|\mb{k}^+|}^+)}\right]^\T\mb{a}_k\right) \right\} - \\
	\qquad \sum_{t=1}^T \left\{ 	- \left(\mb{a}_k^\T\left[\mb{\widetilde{y}}_t^{(k_1^+)} \mid \cdots \mid \mb{\widetilde{y}}_t^{(k_{|\mb{k}^+|}^+)}\right]\right)\left(\Delta_{Z_t}^{-1(\mb{k}^+,\mb{k}^-)}\right) \left(\mb{y}_t^{(\mb{k}^-)} - \mb{A}_{\mb{z}_t}^{(\mb{k}^-,\mb{k}^-)}\mb{\widetilde{Y}}_t^{(\mb{k}^-,\mb{k}^-)}\right) \right\},
\end{array}
\end{equation}
which after some rearranging gives
\begin{equation}
\begin{array}{l}
	\log p(\mb{a}_k \mid \mb{y}_{1:T,} \mb{z}_{1:T}, Z_{1:T}, \{\mb{a}_{k'}\}_{k' \neq k}, \{\Delta_l\} ) \propto \\
	\qquad -\frac{1}{2}\mb{a}_k^\T \left\{ \Sigma_0^{-1} + \sum_{t=1}^T \left[\mb{\widetilde{y}}_t^{(k_1^+)} \mid \cdots \mid \mb{\widetilde{y}}_t^{(k_{|\mb{k}^+|}^+)}\right] \left(\Delta_{Z_t}^{-1(\mb{k}^+,\mb{k}^+)}\right) \left[\mb{\widetilde{y}}_t^{(k_1^+)} \mid \cdots \mid \mb{\widetilde{y}}_t^{(k_{|\mb{k}^+|}^+)}\right]^\T \right\} \mb{a}_k + \\
	\qquad \mb{a}_k^\T \sum_{t=1}^T \left\{ \left[\mb{\widetilde{y}}_t^{(k_1^+)} \mid \cdots \mid \mb{\widetilde{y}}_t^{(k_{|\mb{k}^+|}^+)}\right] \left(\Delta_{Z_t}^{-1(\mb{k}^+,\mb{k}^+)}\right) \left(\mb{y}_t^{(\mb{k}^+)}\right) + \right. \\
	\qquad \left. \left[\mb{\widetilde{y}}_t^{(k_1^+)} \mid \cdots \mid \mb{\widetilde{y}}_t^{(k_{|\mb{k}^+|}^+)}\right]\left(\Delta_{Z_t}^{-1(\mb{k}^+,\mb{k}^-)}\right) \left(\mb{y}_t^{(\mb{k}^-)} - \mb{A}_{\mb{z}_t}^{(\mb{k}^-,\mb{k}^-)}\mb{\widetilde{Y}}_t^{(\mb{k}^-,\mb{k}^-)}\right) \right\}.
\end{array}
\end{equation}
Before completing the square, we will find it useful to introduce a bit more notation to simplify the expression,
\begin{equation*}
	\mb{\bar{Y}}_t^{(\mb{k}^+)} = \left[\mb{\widetilde{y}}_t^{(k_1^+)} \mid \cdots \mid \mb{\widetilde{y}}_t^{(k_{|\mb{k}^+|}^+)}\right], \qquad	\mb{\epsilon}_t^{(\mb{k}^-)} = \mb{y}_t^{(\mb{k}^-)} - \mb{A}_{\mb{z}_t}^{(\mb{k}^-,\mb{k}^-)}\mb{\widetilde{Y}}_t^{(\mb{k}^-,\mb{k}^-)},
\end{equation*}
yielding
\begin{equation}
\begin{array}{l}
	\log p(\mb{a}_k \mid \mb{y}_{1:T,} \mb{z}_{1:T}, Z_{1:T}, \{\mb{a}_{k'}\}_{k' \neq k}, \{\Delta_l\} ) \propto \\
	\qquad -\frac{1}{2}\mb{a}_k^\T \left\{ \Sigma_0^{-1} + \sum_{t=1}^T \mb{\bar{Y}}_t^{(\mb{k}^+)} \Delta_{Z_t}^{-1(\mb{k}^+,\mb{k}^+)} \mb{\bar{Y}}_t^{\T(\mb{k}^+)} \right\} \mb{a}_k + \\
	\qquad \mb{a}_k^\T \left\{ \sum_{t=1}^T  \mb{\bar{Y}}_t^{(\mb{k}^+)}\left( \Delta_{Z_t}^{-1(\mb{k}^+,\mb{k}^+)} \mb{y}_t^{(\mb{k}^+)} + \Delta_{Z_t}^{-1(\mb{k}^+,\mb{k}^-)} \mb{\epsilon}_t^{(\mb{k}^-)}\right)\right\}.
\end{array}
\end{equation}
We desire an expression in the form $-\frac{1}{2}(\mb{a}_k - \mb{\mu}_k)^\T \Sigma^{-1}_k (\mb{a}_k - \mb{\mu}_k)$ for unknown $\mb{\mu}_k$ and $\Sigma^{-1}_k$ so that it conforms to the multivariate normal density with mean $\mb{\mu}_k$ and precision $\Sigma^{-1}_k$. We already have our $\Sigma^{-1}_k$ value from the quadratic term above,
\begin{equation}
	\Sigma^{-1}_k = \Sigma_0^{-1} + \sum_{t=1}^T \mb{\bar{Y}}_t^{(\mb{k}^+)} \Delta_{Z_t}^{-1(\mb{k}^+,\mb{k}^+)} \mb{\bar{Y}}_t^{\T(\mb{k}^+)},
\end{equation}
which allows us to solve the cross-term for $\mb{\mu}_k$,
\begin{align}
	-\frac{1}{2}(-2\mb{\mu}_k^\T\Sigma_k^{-1}\mb{a}_k) & = \mb{a}_k^\T \left( \sum_{t=1}^T  \mb{\bar{Y}}_t^{(\mb{k}^+)}\left(\Delta_{Z_t}^{-1(\mb{k}^+,\mb{k}^+)} \mb{y}_t^{(\mb{k}^+)} + \Delta_{Z_t}^{-1(\mb{k}^+,\mb{k}^-)} \mb{\epsilon}_t^{(\mb{k}^-)} + \right)\right), \nonumber \\
	\Sigma_k^{-1}\mb{\mu}_k & =  \sum_{t=1}^T  \mb{\bar{Y}}_t^{(\mb{k}^+)}\left(\Delta_{Z_t}^{-1(\mb{k}^+,\mb{k}^+)} \mb{y}_t^{(\mb{k}^+)} + \Delta_{Z_t}^{-1(\mb{k}^+,\mb{k}^-)} \mb{\epsilon}_t^{(\mb{k}^-)} + \right).
\end{align}
We can pull the final required $-\frac{1}{2}\mb{\mu}_k^\T \Sigma_k^{-1} \mb{\mu}_k$ term from the proportionality and complete the square. Thus, we have the form of the posterior for $\mb{a}_k$,
\begin{align}
	p(\mb{a}_k \mid \mb{y}_{1:T,} \mb{z}_{1:T}, Z_{1:T}, \{\mb{a}_{k'}\}_{k' \neq k}, \{\Delta_l\} ) & \propto \exp\left( -\frac{1}{2}(\mb{a}_k - \mb{\mu}_k)^\T \Sigma^{-1}_k (\mb{a}_k - \mb{\mu}_k)\right) \nonumber \\
	& \propto \Ncal(\mb{\mu}_k, \Sigma_k),
\end{align}
where
\begin{align}
\begin{aligned}
	\Sigma_k^{-1} & = \Sigma_0^{-1} + \sum_{t=1}^T \mb{\bar{Y}}_t^{(\mb{k}^+)} \Delta_{Z_t}^{-1(\mb{k}^+,\mb{k}^+)} \mb{\bar{Y}}_t^{\T(\mb{k}^+)} \\
	\Sigma_k^{-1}\mb{\mu}_k & =  \sum_{t=1}^T  \mb{\bar{Y}}_t^{(\mb{k}^+)}\left(\Delta_{Z_t}^{-1(\mb{k}^+,\mb{k}^+)} \mb{y}_t^{(\mb{k}^+)} + \Delta_{Z_t}^{-1(\mb{k}^+,\mb{k}^-)} \mb{\epsilon}_t^{(\mb{k}^-)}\right).
\end{aligned}
\end{align}

\section{Experiment Parameters}
\label{sec:exp_params}
Below, we give explicit values for the various priors and parameters used in our experiments. 

\begin{table}[ht]
  	\footnotesize
	\centering
	\begin{tabular}{llr}
		\toprule
		parameter & description & value \\ \midrule
		$N$ & number of time series per event & 6 \\
		$r$ & AR model order & 1 \\
		$\mb{m}_0$ & $\mb{a}_k$ $\Ncal$ prior mean & 0 \\
		$\Sigma_0$ & $\mb{a}_k$ $\Ncal$ prior covariance & $0.1\cdot I_{1\times 1}$ \\
		$L$ & truncated number of event states & 20 \\
		$b_0$ & $\Delta_l$ HIW prior degrees of freedom & $N+3$ \\
		$D_0$ & $\Delta_l$ HIW prior scale & $0.05(b_0-N-1)(I_{N\times N} + 1)$ \\
		$(a_{\alpha_e+\kappa_e}, b_{\alpha_e+\kappa_e})$ & $\alpha_e+\kappa_e$ gamma prior & $(1, 1)$ \\
		$(a_{\gamma_e}, b_{\gamma_e})$ & $\gamma_e$ gamma prior & $(1, 1)$ \\
		$(c_{\rho_e}, d_{\rho_e})$ & $\rho_e$ beta prior & $(1, 1)$ \\
		$(a_{\gamma_c}, b_{\gamma_c})$ & $\gamma_{c}$ gamma prior & $(1,1)$ \\
		$(a_{\kappa_c}, b_{\kappa_c})$ & $\kappa_c$ gamma prior & $(1000,1)$ \\
		$\sigma^2_{\gamma_c}$ & $\gamma_c$ Metropolis-Hastings proposal variance & $1$ \\
		$\sigma^2_{\kappa_c}$ & $\kappa_c$ Metropolis-Hastings proposal variance & $100$ \\
		$(a_{\alpha_c}, b_{\alpha_c})$ & $\alpha_c$ gamma prior & $(1,1)$ \\
		\bottomrule
	\end{tabular}
	\caption[Sparse factorial BP-AR-HMM simulation experiment parameters.]{Parameters used in sparse factorial BP-AR-HMM simulation experiment}
	\label{tbl:HIW_BPARHMM_simParams}
\end{table}

\begin{table}[ht]
  	\footnotesize
	\centering
	\begin{tabular}{llr}
		\toprule
		parameter & description & value \\ \midrule
		$N$ & number of time series per event & 16 and 6 \\
		$r$ & AR model order & 5 \\
		$\mb{m}_0$ & $\mb{a}_k$ $\Ncal$ prior mean & $\mb{0}$ \\
		$\Sigma_0$ & $\mb{a}_k$ $\Ncal$ prior covariance & $\mbox{Cov}(\{y_t^{(i)}\}_{\forall t,i})$ \\
		$L$ & truncated number of event states & 30 \\
		$b_0$ & $\Delta_l$ (H)IW prior degrees of freedom & $N+3$ \\
		$D_0$ & $\Delta_l$ (H)IW prior scale & $(b_0-N-1)\cdot \mbox{Cov}(\{\mb{y}_{t+1} - \mb{y}_t\}_{\forall t})$ \\
		$(a_{\alpha_e+\kappa_e}, b_{\alpha_e+\kappa_e})$ & $\alpha_e+\kappa_e$ gamma prior & $(1, 1)$ \\
		$(a_{\gamma_e}, b_{\gamma_e})$ & $\gamma_e$ gamma prior & $(1, 1)$ \\
		$(c_{\rho_e}, d_{\rho_e})$ & $\rho_e$ beta prior & $(1, 1)$ \\
		$(a_{\gamma_c}, b_{\gamma_c})$ & $\gamma_{c}$ gamma prior & $(1,1)$ \\
		$(a_{\kappa_c}, b_{\kappa_c})$ & $\kappa_c$ gamma prior & $(1000,1)$ \\
		$\sigma^2_{\gamma_c}$ & $\gamma_c$ Metropolis-Hastings proposal variance & $1$ \\
		$\sigma^2_{\kappa_c}$ & Metropolis-Hastings proposal variance & $100$ \\
		$(a_{\alpha_c}, b_{\alpha_c})$ & gamma prior & $(1,1)$ \\
		\bottomrule
	\end{tabular}
	\caption[Sparse factorial BP-AR-HMM seizure experiments parameters.]{Parameters used in epileptic seizures and bursts experiments. When applicable, the same parameters were used for the standard BP-AR-HMM as in the correlated BP-AR-HMMs. The analysis of two two seizures involved 16 iEEG channels, and the analysis of the 15 bursts and single seizure involved 6 iEEG channels.}
	\label{tbl:HIW_BPARHMM_szParams}
\end{table}

\section*{Acknowledgements}
This work is supported in part by AFOSR Grant FA9550-12-1-0453 and DARPA Grant FA9550-12-1-0406 negotiated by AFOSR, ONR Grant N00014-10-1-0746, NINDS RO1-NS041811, RO1-NS48598, and U24-NS063930-03, and the Mirowski Discovery Fund for Epilepsy Research.







\end{document}

%% file: etc/mathShortcuts.tex
\newcommand{\mb}[1]{\boldsymbol{\rm #1}} 
\newcommand{\mat}[2]{\left[\begin{array}{#1}#2\end{array}\right]}

\newcommand{\Ncal}{{\cal N}}

\newcommand{\HIW}{\mbox{HIW}}

\newcommand{\Multi}{\mbox{Multi}}

\newcommand{\Dir}{\mbox{Dir}}
\newcommand{\Ber}{\mbox{Ber}}
\newcommand{\Bin}{\mbox{Bin}}

\newcommand{\BP}{\mbox{BP}}
\newcommand{\BeP}{\mbox{BeP}}
\newcommand{\BetaD}{\mbox{Beta}}
\newcommand{\GammaD}{\mbox{Gamma}}

\newcommand{\T}{{\rm T}}

\newtheorem{lemma}{Lemma}[section]